\documentclass{article}

\usepackage{microtype}
\usepackage{graphicx}
\usepackage{subfigure}
\usepackage{booktabs} % for professional tables

\usepackage{hyperref}

\usepackage[accepted]{icml2025}

\usepackage{amsmath}
\usepackage{amssymb}
\usepackage{mathtools}
\usepackage{amsthm}

\usepackage[capitalize,noabbrev]{cleveref}

\theoremstyle{plain}
\newtheorem{theorem}{Theorem}[section]

\newtheorem{lemma}[theorem]{Lemma}

\theoremstyle{definition}

\newtheorem{assumption}[theorem]{Assumption}
\theoremstyle{remark}

\usepackage[textsize=tiny]{todonotes}
\usepackage{svg}
\usepackage{booktabs}
\usepackage{multirow, multicol}
\usepackage{tablefootnote}
\usepackage{hyperref}
\usepackage{makecell}
\usepackage{listings}
\definecolor{codegreen}{rgb}{0.313, 0.498, 0.498} %{0,0.6,0}
\definecolor{codegray}{rgb}{0.5,0.5,0.5}
\definecolor{codepurple}{rgb}{0.58,0,0.82}
\definecolor{backcolour}{rgb}{0.95,0.95,0.92}
\lstdefinestyle{mystyle}{
    commentstyle=\color{codegreen},
    keywordstyle=\color{magenta},
    numberstyle=\tiny\color{codegray},
    stringstyle=\color{codepurple},
    basicstyle=\ttfamily\footnotesize,
    breakatwhitespace=false,         
    breaklines=true,                 
    captionpos=b,                    
    keepspaces=true,                 
    numbers=left,                    
    numbersep=5pt,                  
    showspaces=false,                
    showstringspaces=false,
    showtabs=false,                  
    tabsize=2
}

\icmltitlerunning{Personalized Federated Fine-tuning for Heterogeneous Data: An Automatic Rank Learning Approach via Two-Level LoRA}

\begin{document}
% PerFiT: 
\twocolumn[
\icmltitle{Personalized Federated Fine-tuning for Heterogeneous Data:\\An Automatic Rank Learning Approach via Two-Level LoRA}
% Personalized Federated Fine-tuning for  Heterogeneous Data: a Two-Level Low Rank Adaptation Approach


\icmlsetsymbol{equal}{*}

\begin{icmlauthorlist}
\icmlauthor{Jie Hao }{yyy}
\icmlauthor{Yuman Wu}{yyy}
\icmlauthor{Ali Payani}{comp}
\icmlauthor{Myungjin Lee}{comp}
\icmlauthor{Mingrui Liu}{yyy}
\end{icmlauthorlist}

\icmlaffiliation{yyy}{Department of Computer Science, George Mason University, Fairfax, VA, USA.}
\icmlaffiliation{comp}{Cisco, Cupertino, CA, USA}

\icmlcorrespondingauthor{Mingrui Liu}{mingruil@gmu.edu}

\icmlkeywords{Machine Learning, ICML}

\vskip 0.3in
]



\printAffiliationsAndNotice{}  % leave blank if no need to mention equal contribution
% they often overlook data heterogeneity and model personalization.  The primary challenge is that a single common adapter or prompt learner may struggle to accommodate the diverse data across all clients.
\begin{abstract}
We study the task of personalized federated fine-tuning with heterogeneous data in the context of language models, where clients collaboratively fine-tune a language model (e.g., BERT, GPT) without sharing their local data, achieving personalization simultaneously. While recent efforts have applied parameter-efficient fine-tuning techniques like low-rank adaptation (LoRA) in federated settings, 
they typically use single or multiple independent low-rank adapters with predefined maximal and minimal ranks, which may not be optimal for diverse data sources over clients.

To address this issue, we propose PF2LoRA, a new personalized federated fine-tuning algorithm built on a novel \emph{automatic rank learning approach via two-level LoRA}. Given the pretrained language model whose weight is frozen, our algorithm aims to learn two levels of adaptation simultaneously: the first level aims to learn a common adapter for all clients, while the second level fosters individual client personalization. A key advantage of PF2LoRA is its ability to adaptively determine a suitable rank based on an individual client’s data, rather than relying on a predefined rank that is agnostic to data heterogeneity. We present a synthetic example that highlights how PF2LoRA automatically learns the ground-truth rank for each client, tailoring the adaptation to match the properties of their individual data. Notably, this approach introduces minimal additional memory overhead, as the second-level adaptation comprises a small number of parameters compared to the first level. Our experiments on natural language understanding and generation tasks demonstrate that PF2LoRA significantly outperforms existing federated fine-tuning methods.
\end{abstract}
% 
% To illustrate this effectiveness of this approach, 
% % This framework explicitly accommodates variations in adapter matrix ranks across clients and
\section{Introduction}

Federated learning (FL) \citep{mcmahan2017communication,kairouz2021advances} is a crucial paradigm for enabling collaborative training of machine learning models across distributed clients while preserving data privacy \citep{mcmahan2017learning,geyer2017differentially}. FL is particularly important in some scenarios that involve sensitive data, such as healthcare~\citep{brisimi2018federated, sheller2020federated},  finance~\citep{yang2019federated}, and mobile devices \citep{bonawitz2019towards}. However, in the context of foundation models like BERT \citep{devlin2018bert} and GPT \citep{radford2018improving}, traditional FL algorithms face significant challenges due to the complexity of these models.
It requires huge computing resources when directly fine-tuning model parameters on the heterogeneous data distributed across different clients.

To address the issue of fine-tuning foundation models, many parameter-efficient fine-tuning (PEFT) methods such as prompt tuning \citep{lester2021power} and low-rank adaptation (LoRA) \citep{hu2021lora} have been explored, where LoRA freezes the original pre-trained parameters $W\in \mathbb{R}^{m\times n}$ of the foundation model while fine-tuning additional low rank matrices $B\in\mathbb{R}^{m\times r}$ and $A\in \mathbb{R}^{r\times n}$, $ r\ll \min(m, n)$. This technique enables fine-tuning large models with a reduced number of trainable parameters, making them more suitable for resource-constrained devices. This paper specifically focuses on LoRA in the context of federated learning for heterogeneous data.

\begin{figure}[!t]
    \centering
    \includesvg[scale=0.23]{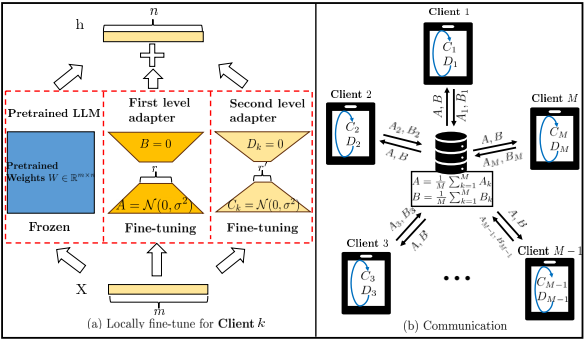}
    \vspace*{-0.1in}
     \caption{Overview of the two-level low-rank adaptation framework. The first level learns a common adapter $\{A, B\}$ for all clients, and the common adapter is synchronized by averaging across all the clients at every communication round.
     The second level aims to learn a client-specific and lightweight adapter $\{C_k, D_k\}$ for a specific client $k\in[1, M]$,  while the lightweight adapters introduce negligible additional memory overhead.}   \vspace*{-0.15in} 
    \label{fig:illustration}
\end{figure}

A natural method to perform low rank adaptation in federated learning is to adopt the same rank $r$ of matrices $A$ and $B$ across different clients. This method is referred to as homogeneous LoRA (HOMLoRA), but it does not accommodate the personalized requirement of clients with heterogeneous data distributions. Recent work HETLoRA \citep{cho2024heterogeneous} highlights the importance of heterogeneous rank configurations to enable personalized federated learning, which proposed ``matrix truncation", ``local rank self-pruning", and ``sparsity-weighted aggregation"  to learn various ranks $r_k$ for  the heterogeneous data from clients. However, this approach suffers from two main drawbacks: (1) The initial rank for any client is fixed and in the range of predefined minimal and maximal ranks, which is independent of client data. However, it is possible that the clients learning difficult tasks are assigned with smaller ranks and do not have the capacity to learn their corresponding tasks well. We empirically observe this issue and analysis theoretically the reason in Section \ref{sec:automatic_rank_learning}. (2) There are many hyperparameters which need to be tuned, including the minimal and maximal values of rank, the pruning parameter, and the sparsity parameter. It remains unclear how to perform personalized federated fine-tuning such that the adapter is dependent on the data and the procedure has a small number of tuning parameters.
%(3) The dynamical rank pruning, matrix truncation and zero-padding directly changes the intrinsic structure of local adapters, which leads to unstable training.
% $r_{min}, r_{max}$ $\gamma$ $\lambda$

In this paper, we propose PF2LoRA, a novel personalized federated fine-tuning algorithm that explicitly incorporates heterogeneous ranks into the problem formulation. Our approach introduces a \emph{two-level low-rank adaptation framework}. The first level learns a common adapter shared among all clients with trainable parameters $x=\{B\in\mathbb{R}^{m\times r}, A\in\mathbb{R}^{r\times n}\}$, while the second level enables client-specific personalization by learning lightweight, client-specific adapter $y_k$ for $k$-th client, defined as $y_k=\{ D_k\in\mathbb{R}^{m\times \tilde{r}}, C_k\in\mathbb{R}^{ \tilde{r}\times n}, 0<\tilde{r}<r\}$ and $1\leq k \leq M$ ($M$ represents the number of participating clients). We formulate the two-level low-rank adaptation framework as a bilevel optimization problem, aiming to learn a common adapter $x$ that minimizes the loss function given the fact the individual client adapters $\{y_k\}_{k=1}^{M}$ can achieve the optimal performance when conditioned on the shared adapter $x$. The two-level LoRA framework explicitly accommodates variations in adapter matrix ranks across clients, i.e. $r-\tilde{r}\leq r_k \leq  r+\tilde{r}$. That allows the algorithm to automatically learn the ground-truth rank for each client based on their data heterogeneity.   

Thus our algorithm essentially circumvents the rank pruning, matrix truncation, and zero-padding in HETLoRA for the alignment of adapters. 
Besides, the whole framework increases negligible additional memory overhead, as the second-level low rank adaptation comprises a small number of parameters compared to the first level. 
Our main contribution is listed as follows:
\begin{itemize}
    \item We propose a novel bilevel formulation for personalized fine-tuning on heterogeneous data, and develop a two-level low rank adaptation framework to efficiently fine-tune foundation model in the scenario of federated learning. The main workflow of our framework is illustrated in Figure \ref{fig:illustration}. 
    \item We provide a synthetic example explaining why HETLoRA fails to learn the ground truth rank of clients, resulting in underfitting in a multivariate linear regression example. Then we conducted an experiment on personalized federated fine-tuning with two clients. The experimental results demonstrate that our algorithm can automatically learn the ground-truth of clients' rank to accomodate the data heterogeneity. 
    \item Through extensive experiments on various natural language understanding and generation tasks, we demonstrate that PF2LoRA significantly outperforms existing federated fine-tuning baselines, providing a robust and efficient solution for personalized federated learning with foundation models.  For example on GLUE benchmark, PF2LoRA achieves $25.6\%$, $2.33\%$, $15.34\%$, and $2.53\%$ higher performance than state-of-the-art baseline HETLoRA on MNLI, SST-2, QQP, QNLI dataset, respectively. In addition, through extensive ablation studies, we show that our proposed two-level adaptation framework achieves the highest performance across various data heterogeneity levels and outperforms baseline methods even if they use more trainable parameters.
\end{itemize}
\vspace*{-0.1in}
\section{Related Work}
\noindent\textbf{Parameter-efficient Fine-Tuning.} There are various categories of parameter-efficient fine-tuning (PEFT) techniques, where only a subset of parameters of the pretrained foundation model or a small number of additional parameters are updated to adapt to the target task. The first line of work includes bias update or head-tuning~\citep{lee2019would,zaken2021bitfit,lawton2023neural,wei2021pretrained} and weight masking~\citep{zhao2020masking,sung2021training,xu2021raise}. The second line of work considers adapters~\citep{houlsby2019parameter,he2021towards}, prompt tuning~\citep{lester2021power,li2021prefix} and low rank matrix adaptation~\citep{hu2021lora}. Different from these works, our paper focuses on designing new federated learning algorithms based on low rank adaptation with heterogeneous data, where the local client data is not shared to other clients.

\noindent\textbf{Federated Learning with Fine-tuning.} The PEFT framework has been recently incorporated in the FL framework~\citep{babakniya2023slora,zhang2024towards,zhang2023fedpetuning,cho2024heterogeneous,wang2023can}. However, most of them do not consider the data heterogeneity in the context of foundation models. To the best of our knowledge, HETLoRA~\citep{cho2024heterogeneous} is the only work which allows data-independent heterogeneous ranks for each clients by a fixed rank initialization, zero-padding, truncation, self-pruning and sparsity regularization. In contrast, our work promotes data-dependent heterogeneous ranks of local clients by an explicit bilevel modeling and reduce the number of tuning hyperparameters.

\vspace*{-0.1in}
\section{Preliminaries}
In this section, we introduce a few parameter-efficient fine-tuning methods in the context of (federated) foundation model learning. It includes LoRA, HOMLoRA, HETLoRA. Due to limited space, we describe a variant of the personalized federated learning algorithm in the context of low rank adaptation method, namely Per-FedAvg-LoRA, in Appendix~\ref{app:per-fedavg}.

\noindent\textbf{Low-rank adaptation (LoRA).} LoRA is a technique designed to efficiently fine-tune large pre-trained models by injecting trainable low-rank matrices into each layer of a foundation model \citep{hu2021lora}.
% This approach significantly reduces the number of trainable parameters compared to full fine-tuning, making it suitable for scenarios with limited computational resources.
Formally, consider a pre-trained model where the original weight matrix is denoted as $W_0 \in \mathbb{R}^{m \times n}$. The model update $\Delta W$ during the fine-tuning can be approximated by multiplication of two low-rank matrices $B\in\mathbb{R}^{m \times r}$ and $A\in \mathbb{R}^{r \times n}$.
The updated weight matrix $W$ is then given by:
  \vspace*{-0.05in}
\begin{equation}\label{eq:w_update}
    \begin{aligned} 
        W = W_0 + \Delta{W} = W_0 + BA.
    \end{aligned}
\end{equation}
This decomposition allows the model to learn adaptations for down-stream tasks while keeping the majority of the original weights frozen, thereby maintaining the pre-trained knowledge and significantly reducing memory and computational overhead.

\noindent\textbf{HOMLoRA}. \label{sec:HOMLoRA} When considering LoRA in the scenario of federated learning, a natural extension is refereed to as HOMLoRA, which adopts homogeneous rank $r$ across all the clients. Assume that $M$ clients participate in federated learning at every communication round. The objective function of each client $k$ is $f_k(\cdot)$, and the goal is to find a common adapter $x =\{A\in \mathbb{R}^{m\times r}, B\in\mathbb{R}^{r\times n} \}$ that performs well across all the clients. It aims to solve the optimization problem: $\min_{x} \frac{1}{M}\sum_{k=1}^{M}f_k(x)$.
% \begin{equation}
%     \min_{x} \frac{1}{M}\sum_{k=1}^{M}f_k(x)
% \end{equation}
Specifically, each client locally updates their adapters for $I$ steps by Adam (or SGD) using their local data, and the server aggregates the adapters from each local clients $\{A_k^{t}, B_k^{t}\}_{k=1}^{M}$ ($k$ is the local client id) at iteration $t$ when $t$ is a multiple of $I$, where $I$ is the number of local updates per round: $A^{t} = \frac{1}{M}\sum_{k=1}^{M}A^{t}_k, \ B^{t} = \frac{1}{M}\sum_{k=1}^{M}B^{t}_k$. Then the server broadcasts the aggregated adapters back to each client. HoMLoRA can be regarded as a direct extension of FedAvg~\citep{mcmahan2017communication} in the context of LoRA~\citep{hu2021lora}.

\noindent \textbf{HETLoRA.} \label{sec:HETLoRA} Recently,  \citet{cho2024heterogeneous} proposed a heterogeneous LoRA method, namely HETLoRA, which is able to learn heterogeneous low rank matrices for different clients. The main technical components contain four parts: (1) a fixed rank initialization: where the $r_k$ is fixed for $k$-th client and $r_{min}\leq r_k \leq r_{max}$; (2) distribution via truncation, wherein at each communication round, 
each client truncates the global adapter matrices to align dimensions $A_k^t=A_{:r_k,:}^t, B_k^t=B_{:, :r_k}^t$;

(3) local training with self-pruning, which introduces the regularization term (with a penalty coefficient $\lambda$) to induce adapter sparsity (with a sparsity factor $\gamma$), and it dynamically reduces the $r_k$ by pruning unimportant columns in $B_k^t$ (or rows in $A_k^t$); (4) sparsity-weighted aggregation, wherein each communication round, to aggregate the adapter matrices with different rank $r_{min}\leq r_k \leq r_{max}$, the server reconstructs $\{A_k^t, B_k^t\}$ by zero-padding them to $r_{max}$.

Then HETLoRA updates the common adapter by aggregating the local adapters with an aggregation weight. In particular, the update rule is $A^{t+1} = \sum_{k=1}^{M} \|\Delta W_k^t\|A^{t}_k/\sum_{k=1}^{M}\|\Delta W_k^t\|$  and $B^{t+1} =  \sum_{k=1}^{M} \|\Delta W_k^t\|B^{t}_k/\sum_{k=1}^{M}\|\Delta W_k^t\|$, $\Delta W_k^t=B_k^t A_k^t$.

However, the performance of HETLoRA heavily depends on (1) the fixed rank initialization, which is independent of data and may cause underfitting or overfitting issues, and (2) the proper setting for a set of hyperparameters, including $r_{\min}$, $r_{\max}$, $\gamma$, and $\lambda$. 

To address these issues, we propose a new two-level low-rank adaptation framework for personalized fine-tuning in the next subsection.

\section{A New Two-level Adaptation for Personalized Federated Fine-Tuning}
As we discussed in Section~\ref{sec:HOMLoRA}, HOMLoRA uses only one common adapter $x=\{B\in\mathbb{R}^{m\times r}, A\in\mathbb{R}^{r\times n}\}$ across all the clients, which is insufficient to learn from the heterogeneous data in federated learning. Therefore, we introduce a client-specific adapter for every client $k$ with $y_k=\{ D_k\in\mathbb{R}^{m\times \tilde{r}}, C_k\in\mathbb{R}^{ \tilde{r}\times n}, 0<\tilde{r}<r, 1\leq k \leq M\}$. We emphasize that the newly introduced adapter has a much smaller rank $\tilde{r}$ than that in the common adapter. Empirically, we usually set $\tilde{r} = \frac{r}{4}$ or $\frac{r}{2}$, which means the trainable parameters in the client-specific adapter are only $\frac{1}{4}$ or $\frac{1}{2}$ of those in the common adapter. Thus the new adapter is lightweight and incurs negligible additional memory overhead. 

Different from \eqref{eq:w_update}, we incorporate both the common and client-specific adapters. In particular, the adapter for the $k$-th client can be parameterized as, 
\vspace*{-0.05in}
\begin{equation}\label{eq:origin_update}
    \begin{aligned}
        W_k = W_0 + BA + D_kC_k,
    \end{aligned}
\end{equation}
where $W_k$ is the adapter for $k$-th client, $A,B$ are common adapters for all clients, and $(C_k,D_k)$ are client-specific adapters for $k$-th client.
Since the original weight $W_0$ is frozen, the trainable parameters in the model are $A, B, C_k, D_k$ for the client $k$. Different than the HETLoRA whose local client matrix rank is predefined and independent of data, our specific parameterization~\eqref{eq:origin_update} explicitly encourages each adapter $W_k$ for the $k$-th client to vary over $k$: it can have different ranks in the range $(r-\tilde{r}, r+\tilde{r})$ and the specific rank is automatically determined by the training data.

We formalize our two-level adaptation framework for personalized federated fine-tuning as the following bilevel optimization problem:

\vspace*{-0.15in}
\begin{equation}\label{eq:bilevel_obj}
    \begin{aligned}
        &\min_{x} \Phi(x)\coloneq \frac{1}{M}\sum_{k=1}^{M}f_k(x, y^*_{k}(x)),  & (\text{UL})\\
        & \text{s.t.}, \ y^*_k(x) \in \arg\min_{y_k} f_k(x, y_k),  &(\text{LL})
    \end{aligned}
\end{equation}
where $f_k(x, y_{k}) \coloneq \mathbb{E}_{\xi\sim \mathcal{D}_k} F_k(x, y_k; \xi)$ is the loss function for the $k$-th client, $F_k$ the individual loss function for a sample $\xi$ from the $k$-th client, and $\mathcal{D}_k$ is the data on client $k$. The upper-level (UL) learns a common adapter $x$ for all the clients upon a set of the best client-specific adapters $\{y_k^*(x)\;|\;1\leq k\leq M\}$ for given $x$ defined by the lower-level problem. Given the common adapter, the lower-level (LL) aims to locally search the optimal client-specific adapter to fit its respective data, which in fact
fosters individual client personalization.  

%Therefore, our proposed method is essentially a two-level  low rank adaptation. 

\begin{algorithm}[!t]
    \caption{\textsc{Two-level Adaptation for Personalized Fine-Tuning}} \label{alg:PF2LORA}
    \begin{algorithmic}[1]
        \STATE \textbf{Input:} $ \alpha, \eta, I, T, M, \mathcal{D}_k$ 
        % \STATE \textbf{Initialize:} $x^0_k, y^0_k$
        \FOR {$k \in \{1,\ldots,M\}$ \textbf{in parallel}} 
        \FOR{$t=0, 1, \dots, T-1$}
            \STATE Sample $\pi_k^t$, $\xi_k^t$,  $\tilde{\xi}_k^t$, $\zeta_k^t$ independently from distribution $\mathcal{D}_k$
                        % \STATE Sample $\pi_t$, $\zeta_t$, $\zeta'_t$ independently from distribution $\gD_g$
            \STATE $y^{t+1}_{k} = y^t_{k}  - \alpha\nabla_y F_k(x^t_k,y^t_k;\pi_k^t)$
            % \STATE $x_{t+1}^k = x_t^k - \eta \nabla_x f_{k}(x_{t}^{k}, y_{t+1}^{k}; \xi_t)(I - \alpha\nabla_{xy}f_{k}(x_{t}^k, y_{t}^{k}; \zeta_t))$
            \STATE $x^{t+1}_k = x^t_k - \eta \nabla_x F_{k}(x^{t}_{k}, y^{t+1}_{k}; \xi_k^t) + \alpha \eta\nabla_{xy}F_{k}(x^{t}_k, y^{t}_{k}; \zeta_k^t)\nabla_yF_{k}(x^t_k, y^{t+1}_{k};\tilde{\xi}_k^t) $\\
            \IF{$t\% I ==0$} 
                 \STATE $x^{t+1} = \frac{1}{M}\sum_{k=1}^{M}x^{t+1}_k$
                 \STATE $x^{t+1}_k = x^{t+1}$
            \ENDIF
        \ENDFOR
        \ENDFOR
    \end{algorithmic}
\end{algorithm}

\noindent\textbf{Algorithm Design.} Now we consider solving \eqref{eq:bilevel_obj} efficiently in personalized federated learning. At the beginning of every communication round, i.e., $(t\%I=0)$, each client $k$ receives the averaged common adapter $x^{t}_k$ from the server, and starts running its local steps.    
We run one step SGD for the lower-level problem to approximately find the minimizer of the lower-level problem (line 5 in Algorithm \ref{alg:PF2LORA}). %Then we can rewrite the lower-level update as,  

Define $\Phi_k(x)=f_k(x,y_k^*(x))$, then
the gradient of the function $\Phi_k(x_k^t)$ in terms of $x_k^t$, namely hypergradient~\citep{ghadimi2018approximation}, can be calculated by chain rule approximately as follows:
\begin{equation} \label{eq:hypergrad}
\begin{aligned}
    \nabla\Phi_k(x^t_k) &\approx \nabla\widehat{\Phi}_k(x^t_k) =\nabla_x f_{k}(x^t_k, y^{t+1}_{k}) \\
    &- \alpha \nabla_{xy}f_{k}(x^t_k, y^{t}_k)\nabla_yf_{k}(x^t_k, y^{t+1}_{k}), 
    \end{aligned}
\end{equation}
where $\approx$ is due to the fact that $y_k^{t+1}$ is only an approximation to the optimal solution $y_k^*(x_k^t)$. Therefore, we use the stochastic version of $\nabla \widehat{\Phi}_k(x_k^t)$ to update the common adapter $x$ on client $k$, as described in line 6 of Algorithm~\ref{alg:PF2LORA}.

In fact, Adam or AdamW can also be used to update the upper-level  variable based on the stochastic gradient information to replace the SGD update as in line 6.  Empirically, we adopt AdamW as the upper-level optimizer (line 6) and SGD as the lower-level optimizer (line 5) to fine-tune a language model. One can refer to Algorithm \ref{alg:PF2LORA} for more details, where line 5 is used to update the client-specific adapter, line 6 is used to update the common adapter, and line 8 corresponds to the synchronization of the common adapter.

\section{Automatic Rank Learning} \label{sec:automatic_rank_learning}

To clarify this mechanism of ``automatic rank learning of PF2LoRA", as well as the failure reason of HETLoRA, we first construct a multivariate linear regression example and provide a theoretical analysis to demonstrate why our method is able to  learn the ground-truth rank, accommodating the heterogeneity of clients' data, whereas HETLoRA fails. Then we conduct a synthetical experiment to compare two algorithms in federated learning with two clients. The experimental results confirm that our algorithm is able to learn and converge to the optimal solution. In contrast, HETLoRA underestimates the initial rank of some clients due to random rank initialization strategy, resulting in underfitting and suboptimal performance in such clients.

Consider a multivariate linear regression in federated learning, 
\begin{equation*}
    \min_{W}\sum_{k=1}^{2}\|X_kW-Y_k\|_F^2 
\end{equation*}
where $(X_k, Y_k)$ is the client-$k$'s data, $W$ is a low-rank matrix and can be decomposed into low-rank matrices, $W=AB$. The details of synthetic experiments are described as follows,
\vspace{-0.15in}
 \paragraph{Ground truth of trainable parameters.} Given two clients, suppose that we have two optimal solutions with low-rank structure,
    $$W_1^*=A_1^*B_1^*, \ s.t., W_1^*\in \mathbb{R}^{10\times 10}, A_1^*\in \mathbb{R}^{10\times 3}, B_1^*\in \mathbb{R}^{3\times 10},$$
    $$W_2^*=A_2^*B_2^*, \ s.t., W_2^*\in \mathbb{R}^{10\times 10},  A_2^*\in \mathbb{R}^{10\times 4}, B_2^*\in \mathbb{R}^{4\times 10},$$
    with $rank(W_1^*)=3$,  $rank(W_2^*)=4$. We initialize the random matrices $A_1^*, B_1^*, A_2^*, B_2^*\sim \mathcal{N}(0, 1)$\footnote{Note that we use $X\sim\mathcal{N}(0,1)$ to denote each entry of the matrix $X$ follows a standard Gaussian distribution.}.
\vspace{-0.15in}

 \paragraph{Training and testing data.}  We construct the synthetic data $(X, Y)$ for two clients respectively by randomly generating 1000 samples, i,e., $X_1\in\mathcal{R}^{1000\times10}, s.t., X_1\sim \mathcal{N}(0, 1)$, $X_2\in\mathcal{R}^{1000\times10}, s.t.,X_2\sim \mathcal{N}(0, 1)$, and their element targets,
    \begin{align*}
        & y_1 = x_1W_1^* + \epsilon_1, \  \epsilon_1\sim \mathcal{N}(0, 0.1), \\
        & y_2 = x_2W_2^* + \epsilon_2, \  \epsilon_2\sim \mathcal{N}(0, 0.2).
    \end{align*}
    The first 700 samples serve as the training set $\mathcal{D}_{k}^{tr}, k=1,2$ and the remaining serves as the testing set $\mathcal{D}_{k}^{te}, k=1,2$.    
\vspace{-0.15in}
\paragraph{Training process.} 
 HETLoRA: Following its rank initialization strategy $r_{min}\leq rank_1 \leq rank_2 ...\leq rank_k...\leq r_{max}$, we assume that $r_{min}=1, r_{max} = 12$ and initialize $\hat{W}_k = \hat{A}_k\hat{B}_k$ by,
\begin{align*}
    & \hat{A}_1\in \mathbb{R}^{10\times 2}, \hat{B}_1\in \boldsymbol{0}^{2\times 10}, \ s.t. \ \hat{A}_1 \sim  \mathcal{N}(0, 1),\\
    & \hat{A}_2\in \mathbb{R}^{10\times 10}, \hat{B}_2\in \boldsymbol{0}^{10\times 10}, \ s.t. \ \hat{A}_2 \sim  \mathcal{N}(0, 1)
\end{align*}
%\vspace*{-0.in}
so we have $rank(\hat{A}_1)=2$ and $rank(\hat{A}_2)=10$. We can easily get that the total number of trainable parameters for two clients is $240$. 
% We search the regularization factor $\gamma$ in the range $[0.05, 0.5\}$ with the search grid $0.05$ and set it to the optimal value $0.1$. The pruning parameter $\gamma=0.3$, which is responsible for imposing the regularization to the last $30\%$ columns to sparse them. We tune the learning rate within the range $\{0.001, 0.002, 0.003, 0.004, 0.005\}$ and set it to the optimal value $0.002$.  The total training steps are 2000, and the communication is performed every 10 steps, which means we train the parameters for 10 steps locally, and then execute the parameter aggregation and distribution.  

PF2LoRA: For a fair comparison, we initialize the trainable parameters $\hat{W}_k = \hat{A}_k\hat{B}_k+\hat{C}_k \hat{D}_k$, and make sure the total number of trainable parameters to be the same as that in HETLoRA.
For client $k=1, 2$, we have $r=4, \tilde{r}=2$ and,
\begin{align*}
    &\hat{A}_k\in \mathbb{R}^{10\times 4}, \hat{B}_k\in \mathbb{R}^{4\times 10}, \hat{C}_k\in \mathbb{R}^{10\times 2}, \hat{D}_k\in \mathbb{R}^{2\times 10}, \\
    & s.t. \ \hat{A}_k, \hat{C}_k, \hat{C}_k, \hat{D}_k \sim  \mathcal{N}(0, 1).
\end{align*}
and $A_kB_k$ is orthogonal to the matrix $C_kD_k$, such that their column space or row space are independent mutually.
The total number of training steps are fixed as 2000, and the communication interval is 10. The details of hyperparameter settings, inlcuding learning rate, pruning parameter of HETLoRA etc., are summarized in Appendix \ref{sec:hyperparameter_syn_exp}. 

% We search the best upper-level and lower-level learning rates within the range $[0.001, 0.01]$ with the search grid  of $0.001$, and set the best upper-level learning rate to $0.005$ and the lower-level learning rate to $0.002$. In each communication round, we aggregate the common adapter parameters $A_k, B_k$ and then distribute them, and the local adapter parameters  $C_k, D_k$ are not involved in communication.   

\paragraph{Evaluation.}
We evaluate the trained model every communication round on the testing data $\mathcal{D}_k^{te}$, and measure the distance between $\hat{W}_k$ and $W_k^*$ by $\|\hat{W}_k - W_k^*\|_F^2$. In addition, we record the rank evolution of two clients as the training steps. 
For PF2LoRA, We compute the singular value $\{\lambda_i|i=1,..., 10\}$ of $\hat{W}_k$ by singular value decomposition (SVD) and determine its rank: $\min_{1\leq j\leq 10} \sum_{i=1}^{j} \lambda_i \geq 0.9 \sum_{i=1}^{10} \lambda_i$, where $\{\lambda_i\}$ keeps descending order. The comparison results, including training, testing loss, frobenius norm distance and rank evolution are shown in Figure \ref{fig:synthetic_exp}.

In the last column of Figure 2(b), PF2LoRA successfully learns the ground truth rank of $3$ for client 1 and $4$ for client 2, demonstrating its ability to automatically learn ranks within the range $[r-\tilde{r}, r+\tilde{r}]$. The training and testing losses decrease rapidly, converging near zero, and the distance to the ground truth parameters consistently reduces to a small value. In contrast, HETLoRA fails to learn the ground truth rank for client 1 due to random initialization, which underestimates the rank and prevents it from covering the true rank. Rank pruning further reduces client 1's rank to $r_{min}=1$, increasing training/testing losses and Frobenius norm distance. However, client 2, with a better initial rank, successfully learns the ground truth rank through pruning. Refer to Appendix \ref{sec:analysis_rank_learning} for theoretical analysis.

    % As you can see in the last column of Figure 2 (b), PF2LoRA can learn the ground truth rank of $3$ in client 1 and $4$ in client 2, which verifies that our algorithm can automatically learn the rank in the range $[r-\tilde{r}, r+\tilde{r}]$. The training and testing loss decrease rapidly and converge to small values close to 0. The distance to the ground truth parameters also decreases consistently to a small value. Instead, HETLoRA fails to learn the ground truth rank in the first client. Specifically, the first client underestimates the initial rank due to the random initialization strategy, such that it cannot cover the ground truth rank. Rank pruning further reduces the first client's rank to $r_{min}=1$, leading to increasing the training, testing loss and Frobenius norm distance. Since the second client has reasonable rank initialization, it is able to learn the ground truth rank by rank pruning.       
    
    \begin{figure*}[!h]
        \centering
        \subfigure[HETLoRA fails to converge to the ground truth.]{\includegraphics[width=0.495\linewidth]{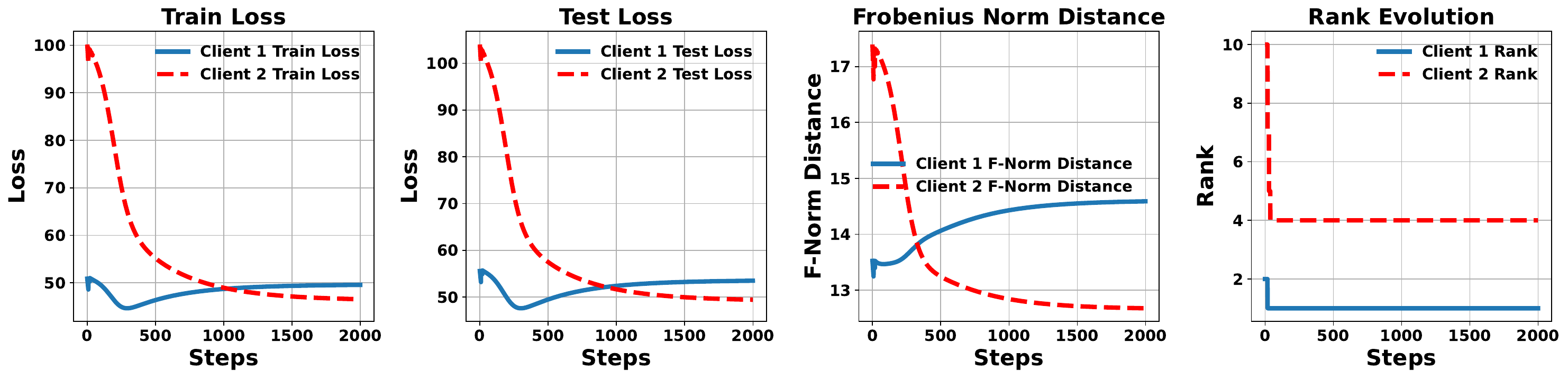}}
        \subfigure[PF2LoRA can converge to the ground truth.]{\includegraphics[width=0.495\linewidth]{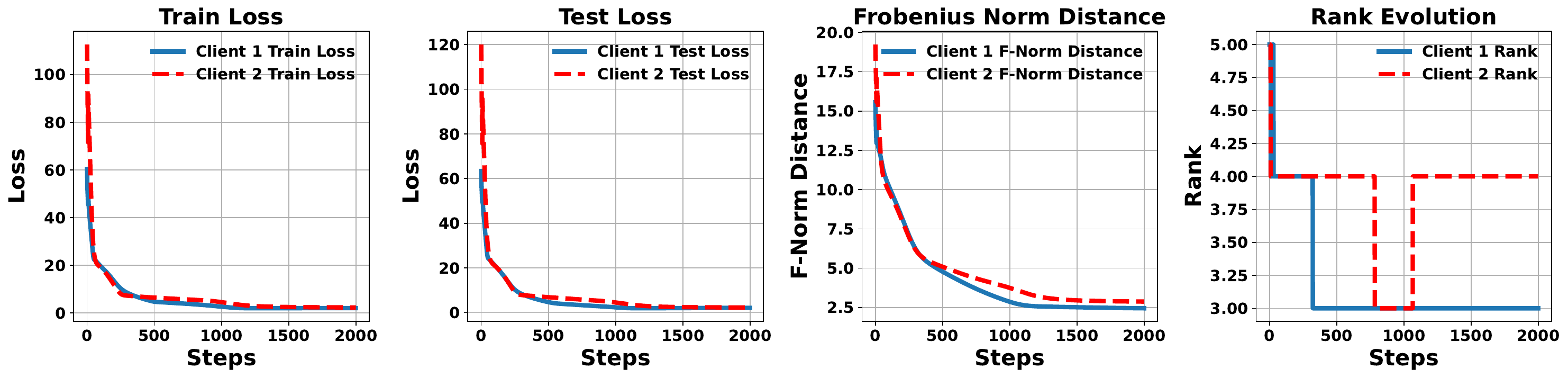}}
        \caption{Comparison of two algorithms. Left to right: the training loss on two clients, the testing loss on two clients, Frobenius norm distance $\|W_k - W_k^*\|_F$, $k = 1, 2$, and the rank evolution of two clients.}
        \label{fig:synthetic_exp}
        \vspace{-0.1in}
    \end{figure*}

\vspace*{-0.1in}
\section{Experiments}
We evaluate PF2LoRA and baseline methods on two major NLP tasks: natural language understanding (NLU) and natural language generation (NLG), where NLU experiments include the text classification on GLUE benchmark \citep{wang2018glue} and question answering task on SQuAD v1 \citep{rajpurkar2016squad}  and v2 \citep{rajpurkar2018know}. NLG experiments are performed on E2E NLG Challenge dataset \citep{novikova2017e2e} and WebNLG dataset \citep{gardent2017webnlg}. Then we execute the ablation studies to explore (1) the performance comparison when other baselines have more trainable parameters than ours in section \ref{sec:basline_more_params};  (2) the impact of data heterogeneity in Appendix \ref{sec:het_level}; and (3) the role of bilevel optimization in our framework in Appendix \ref{sec:ablation_BO}. Training stability is further analyzed in Appendix \ref{sec:stability}. Baseline methods include Centralized LoRA, Homogeneous LoRA (HOMLoRA), Personalized Federated Average LoRA (Per-FedAvg-LoRA), and Heterogeneous LoRA (HETLoRA). The parameter sensitivity analysis is deferred to Appendix \ref{sec:sensitivity}, and the computing and communication costs are presented in Appendix \ref{sec:computing_cost}. 

% \subsection{Baselines}
% We compare with four baselines, including Centralized LoRA, Homogeneous LoRA (HOMLoRA), Personalized Federated Average LoRA (Per-FedAvg-LoRA), and Heterogeneous LoRA (HETLoRA). 
% \paragraph{Centralized LoRA} is the vanilla LoRA algorithm \citep{hu2021lora}, and it is fine-tuned on the homogeneous data, which means it can get access to all the datasets. Therefore, it is generally considered as a performance upper bound in the federated learning. Different from Centralized LoRA, the rest baselines are trained on heterogeneous data.
% \paragraph{HOMLoRA} extends the vanilla LoRA to federated learning, and all the clients keep the adapter rank is the same. 
% \paragraph{Per-FedAvg-LoRA} incoporates the typical personalized federated learning algorithm \citep{fallah2020personalized} into LoRA to enable models quickly adapt to heterogeneous data.

% \paragraph{HETLoRA} is a strong baseline which focuses on federate fine-tuning of foundation model on heterogeneous data. 

\subsection{Natural Language Understanding}
\subsubsection{RoBERTa on Text Classification}
\noindent\textbf{Model}.
% In this section, we adopt RoBERTa  \citep{devlin2018bert} model to perform the personalized federated fine-tuning on the NLU task, i.e., text classification of GLUE benchmark. RoBERTa \citep{liu2019roberta} is an enhancement of BERT~\citep{devlin2018bert} designed to improve its performance on general NLU tasks. It is commonly used and remains a competitive performance in NLU. We take the pre-trained RoBERTa base (125M parameters) and RoBERTa large (355M parameters) from HuggingFace library \citep{wolf2020transformers} and apply the LoRA technique to the model. Other baselines inject only common adapters into each attention-layer of the pretrained RoBERTa. In contrast, PF2LoRA injects common adapters and client-specific adapters into the pretrained model. For baselines Centralized LoRA, HOMLoRA, and Per-FedAvg-LoRA, we initialize the rank $r_k=8$ across all the clients. HETLoRA initializes the client rank $r_k$, such that $r_{min} \leq r_1\leq r_2\leq \cdots \leq r_M \leq r_{max}$. In all the experiments, we tune the best $r_{min}$ and $r_{max}$, and initially assign $r_k = r_{min} + \frac{(r_{max} - r_{min})(k-1)}{M}$. PF2LoRA sets the rank $r_k$ of the common adapter to $8$ and the rank $\tilde{r}_k$ of the client-specific adapter to $2$. 
% The number of trainable parameters of RoBERTa base/large corresponding to the initial rank are listed in Table \ref{tab:parameters_roberta} in Appendix \ref{sec:hy_txt_cl}. 
We use RoBERTa \citep{liu2019roberta} as the backbone model for personalized federated fine-tuning on GLUE tasks. Specifically, we fine-tune pre-trained RoBERTa base (125M parameters) and RoBERTa large (355M parameters) using LoRA. Baselines only use common adapters injected into the attention layers, while PF2LoRA introduces both common adapters and lightweight client-specific adapters. The ranks are set as follows:
\begin{enumerate}
    \item For Centralized LoRA, HOMLoRA, and Per-FedAvg-LoRA: rank $r_k=8$ across all clients.
    \item For HETLoRA: rank $r_k$ vary between $r_{min}$ and $r_{max}$, and it initially assigns $r_k = r_{min} + \frac{(r_{max} - r_{min})(k-1)}{M}$.
    \item For PF2LoRA: common adapter rank  $r_k=8$ and client-specific adapter rank $\tilde{r}_k=2$. 
\end{enumerate}
Details of the rank initialization and trainable parameters are summarized in Table \ref{tab:parameters_roberta} in Appendix \ref{sec:hy_txt_cl}. 
We can observe that the number of trainable parameters in PF2LoRA is slightly increased. Note that HETLoRA uses a different rank  for matrices on different clients, leading to different number of trainable parameters in each client, we count the average trainable parameters of the clients. Experimental results regarding baselines with more trainable parameters will be discussed in Section~\ref{sec:ablation}.

\noindent\textbf{Dataset}.\label{sec:roberta_data}
Following the non-i.i.d. partitioning protocol in \citep{karimireddy2020scaffold}, datasets are split into heterogeneous client datasets with a similarity parameter $s\in [0, 1]$. Each client’s local dataset consists of $(100\times s)\%$ i.i.d. samples from the complete dataset, and $100\times (1- s)\%$ of data sorted by label. Five classification datasets are used for text classification, including CoLA, MNLI, SST-2, QQP, QNLI, from GLUE benchmark. The data summary information is presented in Table \ref{tab:glue_benchmark} in Appendix \ref{sec:hy_txt_cl}. 
 
% \begin{table}[htbp]
%   \centering
%   \caption{The summary of GLUE benchmark.}
%     \scalebox{0.8}{
%     \begin{tabular}{ccccc}
%     \toprule
%     Corpus & \# Train & \# Test & \# Lable & Metrics \\
%     \midrule
%     CoLA  & 8.5k  & 1k    & 2     & Matthews correlation \\
%     MNLI  & 393k  & 20k   & 3     & Accuracy \\
%     SST-2 & 67k   & 1.8k  & 2     & Accuracy \\
%     QQP   & 364k  & 391k  & 2     & Accuracy \\
%     QNLI  & 108k  & 5.7k  & 2     & Accuracy \\
%     \bottomrule
%     \end{tabular}}%
%   \label{tab:glue_benchmark}%
% \end{table}%

\noindent\textbf{Experiment Details}. \label{sec:robert_exp_detail} We run federated fine-tuning algorithms across 8 clients (NVIDIA A100 GPU), where all the clients participate in the training process, with client 0 also performing parameter aggregation and distribution at every communication round. Centralized LoRA, HOMLoRA and HETLoRA use the AdamW optimizer to update the common adapter. Per-FedAvg-LoRA adopts SGD to implement one-step update and AdamW to update the common adapter. PF2LoRA uses SGD to update the client-specific adapter and AdamW to update the common adapter. The learning rates for all methods are tuned and the best choices of learning rate for each baseline can be found in Table \ref{tab:lr_roberta} in Appendix \ref{sec:hy_txt_cl}. For fair comparison, we keep the batch size $\mathcal{B}=16$, and communication interval $I= 10$ for all the federated baselines. The communication rounds $R$ are set according to the dataset size, $\{$CoLA: 50, MNLI: 300, SST-2: 100, QQP: 300, QNLI: 100$\}$, and we keep the same $R$ for all the baselines in a dataset. 

  \vspace*{-0.15in}
\begin{table}[htbp]
  \centering
  \caption{Roberta-base results on GLUE benchmark. We report ``Matthew's correlation" for CoLA and ``Accuracy" for MNLI, SST-2, QQP and QNLI. Higher value means "better performance".}
    \setlength{\tabcolsep}{2pt}
    \begin{tabular}{l|ccccc}
    % \hline
    \Xhline{1.2pt}
    Method  & CoLA  & MNLI  & SST-2  & QQP   & QNLI \\
    \hline
    Centralized LoRA & 56.85 & 83.48 & 93.58 & 86.97 & 89.70 \\
    % \hline
    HOMLoRA & 50.75 & 70.56 & 92.47 & 79.61 & 85.45 \\
    % \hline
    Per-FedAvg-LoRA & 51.11 & 74.73 & 90.56 & 81.26 & 78.59 \\
    % \hline
    HETLoRA & 53.76 & 73.33 & 93.67 & 81.49 & 91.86 \\
    % \hline
    PF2LoRA  & \textbf{54.19} & \textbf{92.14} & \textbf{95.85} & \textbf{93.99} & \textbf{94.18} \\
    % \hline
    \Xhline{1.2pt}
    \end{tabular}%
  \label{tab:result_roberta_base}%
  % \vspace{-0.2in}
\end{table}%

% \begin{table*}[htbp]
%   \centering
%   \caption{Roberta-large results on GLUE benchmark.}
%     \setlength{\tabcolsep}{8pt}
%     \begin{tabular}{|l|c|c|c|c|c|}
%     \hline
%     Method  & CoLA  & MNLI  & SST-2  & QQP   & QNLI \\
%     \hline
%     Centralized LoRA & 57.32 & 84.71 & 93.67 & 88.43 & 90.27 \\
%     \hline
%     HOMLoRA & 51.71 & 74.51 & 93.33 & 79.76 & 89.63 \\
%     \hline
%     Per-FedAvg-LoRA & 51.20  & 75.68 & 92.64 & 81.83 & 79.49 \\
%     \hline
%     HETLoRA & 54.15 & 76.38 & 94.53 & 82.55 & 92.31 \\
%     \hline
%     PF2LoRA  & 56.25 & 93.37 & 97.36 & 94.02 & 94.79 \\
%     \hline
%     \end{tabular}%
%   \label{tab:result_roberta_large}%
% \end{table*}%
We execute evaluate model on each client's test data and report averaged results. Metrics include ``Matthews's correlation" for CoLA and ``Accuracy" for MNLI,SST-2, QQP, QNLI. Results are presented in Table \ref{tab:result_roberta_base} (RoBERTa base) and Table \ref{tab:result_roberta_large} (RoBERTa large) in Appendix \ref{sec:result_roberta_large}, where the heterogeneity level $s=0.3$ is set for CoLA and $s=0.9$ for MNLI, SST-2, QQP, and QNLI. PF2LoRA outperforms all baselines significantly on MNLI, SST-2, QQP, QNLI, and achieves comparable performance to Centralized LoRA on CoLA while surpassing other federated baselines.

\subsubsection{DeBERTa on Question Answering}
\noindent\textbf{Model}. We use DeBERTa \citep{he2021debertav3}, an enhanced transformer encoder with 86M parameters, for question-answering tasks on SQuAD v1 and v2. DeBERTa improves text understanding compared to BERT and RoBERTa, making it suitable for complex tasks like question answering and sentiment analysis.

\noindent\textbf{Dataset}. SQuAD v1/v2 are reading comprehension dataset with over 100k+/150k+ (v1/v2) question-answering pairs extracted from Wikipedia articles. In SQuAD v1, all answers can be derived from the given passage, while SQuAD v2 includes some questions that do not have an answer in the passage, posing a greater challenge. The training data consists of 442 unique topics, while the test sets include 48 topics for SQuAD v1 and 35 topics for SQuAD v2. To ensure consistency, we uniformly sample $80\%$ of the original training set as the new training set and use the remainder as the test set. Heterogeneous data is constructed based on question topics using the method in Section \ref{sec:roberta_data} with heterogeneity parameter $s=0.5$. Exact Match (EM) score and F1 score are two commonly used metrics to evaluate the quality of answers that models provide.

\noindent\textbf{Experiment Details}. Considering the complexity of the question-answering task, we run federated fine-tuning across 4 clients (NVIDIA A100 GPU) with the same heterogeneity parameter $s=0.5$, communication rounds $R=200$, communication interval $I=10$ for SQuAD v1/v2. The optimizer for different baselines follows the settings in Section \ref{sec:robert_exp_detail}. The batch size $\mathcal{B}$ is fixed as 16 for all the baselines for fair comparison. The best learning rate settings for baselines are listed in Table \ref{tab:squadv1_lr} in Appendix \ref{sec:hy_qa}. The initial rank settings for all the baselines can be found in Table \ref{tab:deberta_hyperparam} in Appendix \ref{sec:hy_qa}. The test results of PF2LoRA and other baselines are shown in Table~\ref{tab:result_squad}. PF2LoRA exhibits the highest EM score and F1 score among all federated baselines. For example, PF2LoRA outperforms the best baseline by $4.08\%$ in terms of EM score and $2.20\%$ in terms of F1 score on SQuAD v1.

\vspace*{-0.15in}
\begin{table}[htbp]
  \centering
  \caption{Deberta-v3 results on SQuAD.}
    % \scalebox{1}{
    \begin{tabular}{l|cc}
    \Xhline{1.2pt}
    \multicolumn{1}{c|}{\multirow{2}[2]{*}{Method }} & SQuAD v1.0 & SQuAD v2.0 \\
          & (EM/F1) & (EM/F1) \\
    \hline
    Centralized LoRA \tablefootnote{Note that the results do not exactly match the LoRA results reported in Table 2 in \citep{zhang2023adaptive}. The reason is that the test data used in our experiment is different and more difficult. The test data is a subset of the original training data, which contains much more topics (442 topics) than that in the original test data (48 topics). } & 68.72/83.36 & 44.56/53.31  \\
    HOMLoRA & 68.57/82.99 & 42.53/52.70 \\
    Per-FedAvg-LoRA & 68.80/83.08  & 43.15/53.16 \\
    HETLoRA & 68.64/83.28  & 44.53/54.69 \\
    PF2LoRA  & \textbf{71.61/85.11} & \textbf{44.95/54.71}  \\
    \Xhline{1.2pt}
    \end{tabular}%
  \label{tab:result_squad}%
  % \vspace{-0.2in}
\end{table}%

\subsection{Natural Language Generation}
% GPT-2 with QA datasets (SQuAD)
% LLaMA
% \vspace{-.0.in}
For NLG tasks, we follow LoRA \citep{hu2021lora} to use GPT-2 medium model for federated fine-tuning on WebNLG and E2E NLG Challenge dataset. 

\noindent\textbf{Model}. GPT-2 \citep{radford2019language} 
 is an advanced language model developed by OpenAI. We use GPT-2 medium with 345M parameters and GPT2-XL with 1.5 Billion parameters for NLG tasks.

 \noindent\textbf{Dataset}. WebNLG dataset is a benchmark for evaluating natural language generation systems, including various domains such as sports, cities, universities, hotels and more. E2E NLG Challenge dataset is a NLG dataset especially focusing on restaurants domain.  It emphasizes generating natural, human-like text from structured data (including attributes like restaurant name, food type, price range and rating). For WebNLG, we find that the text style and feature vary with the domains, so we construct the heterogeneous data based on the entry domains. There are 10 domains in the training set and test set. We split the domain into 8 (the number of clients) groups, and make sure that the domains of training and test set on a client are the same. E2E NLG Challenge dataset collects information of 34 restaurants in the training set and 18 restaurants in the test set. We split all the restaurants into 8 (the number of clients) groups by the name, and make sure that the restaurant names in the test set that a client receives are covered by its training set.

\vspace*{-0.15in}
\begin{table}[!h]
  \centering
  \caption{GPT-2 generation results on WebNLG dataset.}
    \setlength{\tabcolsep}{2pt}
    \begin{tabular}{l|cccc}
    \Xhline{1.2pt}
    Method  & BLEU $\uparrow$ & \multicolumn{1}{l}{MET $\uparrow$} & \multicolumn{1}{l}{TER $\downarrow$} & \multicolumn{1}{l}{ROUGE-L $\uparrow$} \\
    \hline
    Centralized LoRA & 0.6031 & 0.7807 & 0.5900  & 0.4169 \\
    % \hline
    HOMLoRA & 0.5141 & 0.7271 & \textbf{0.5697} & 0.4736 \\
    % \hline
    Per-FedAvg-LoRA & 0.5152 & 0.7219 & 0.5746 & 0.4740 \\
    % \hline
    HETLoRA & 0.5196 & 0.7219 & 0.5746 & 0.4740 \\
    % \hline
    PF2LoRA  & \textbf{0.5261} & \textbf{0.7301} & 0.5733 & \textbf{0.4769} \\
    \Xhline{1.2pt}
    \end{tabular}%
    % \vspace{-0.05in}
  \label{tab:result_webnlg}%
\end{table}%
  \vspace*{-0.15in}
\begin{table}[!h]
  \centering
  \caption{GPT2-XL generation results on WebNLG dataset.}
    \setlength{\tabcolsep}{2pt}
    \begin{tabular}{l|cccc}
    \Xhline{1.2pt}
    Method  & BLEU $\uparrow$ & \multicolumn{1}{l}{MET $\uparrow$} & \multicolumn{1}{l}{TER $\downarrow$} & \multicolumn{1}{l}{ROUGE-L $\uparrow$} \\
    \hline
    HOMLoRA & 0.5768 & 0.7771 & \textbf{0.6103} & 0.3967 \\
    Per-FedAvg-LoRA & 0.5783 & 0.7783 & 0.6157 & 0.3972 \\
    HETLoRA & 0.5763 & 0.7789 & 0.6164 & 0.3922 \\
    PF2LoRA  & \textbf{0.5881} & \textbf{0.7832} & 0.6198 & \textbf{0.3978} \\
    \Xhline{1.2pt}
    \end{tabular}%
    \vspace{-0.05in}
  \label{tab:result_webnlg_GPT2_XL}%
\end{table}%

% \begin{table*}[!h]
%   \centering
%   \caption{GPT-2 generation results on E2E dataset.}
%     \setlength{\tabcolsep}{5pt}
%     \begin{tabular}{|l|c|c|c|c|c|}
%     \hline
%     method  & BLEU $\uparrow$  & \multicolumn{1}{l|}{NIST $\uparrow$} & \multicolumn{1}{l|}{MET $\uparrow$} & \multicolumn{1}{l|}{ROUGE-L $\uparrow$} & \multicolumn{1}{l|}{CIDEr $\uparrow$} \\
%     \hline
%     Centralized LoRA & 0.6833 & 8.5321 & 0.4642 & 0.7046 & 2.4023 \\
%     \hline
%     HOMLoRA & 0.5585 & 7.0986 & \textbf{0.4349} & 0.6095 & 1.8327 \\
%     \hline
%     Per-FedAvg-LoRA & 0.5683 & 7.1190 & 0.4327 & 0.6109 & 1.8984 \\
%     \hline
%     HETLoRA &  0.5505     &  7.0088     & 0.4093      & 0.5697      & 1.7167 \\
%     \hline
%     PF2LoRA  & \textbf{0.5717} & \textbf{7.1621} & 0.4321 & \textbf{0.6111} & \textbf{1.9088} \\
%     \hline
%     \end{tabular}%
%   \label{tab:result_e2e}%
% \end{table*}%
\vspace*{-0.15in}
\begin{table}[!h]
  \centering
  \caption{The comparison results with more trainable parameters in baselines. We report "Matthew's correlation" for CoLA and "Accuracy" for MNLI, SST-2, QQP and QNLI. Higher value means "better performance".}
    % setlength{\tabcolsep}{1.5pt}
    \scalebox{0.8}{
    \begin{tabular}{l|cccccc}
    \Xhline{1.2pt}
    Method  & \multicolumn{1}{c}{\# Param} & CoLA  & MNLI  & SST-2 & QQP   & QNLI \\
    \hline
    HOMLoRA   & 0.44M & 52.01 & 73.82 & 92.63 & 80.11 & 86.27 \\
    % \hline
    \small{Per-FedAvg-LoRA}  & 0.44M & 52.35 & 78.62 & 89.65 & 81.12 & 81.41 \\
    % \hline
    HETLoRA    & 0.43M & 53.43 & 79.32 & 94.83 & 81.71 & 92.12 \\
    % \hline
    PF2LoRA   & \textbf{0.37M} & \textbf{54.19} & \textbf{92.14} & \textbf{95.85} & \textbf{93.99} & \textbf{94.18} \\
    \Xhline{1.2pt}
    \end{tabular}}%
  \label{tab:more_prams_result}%
  % \vspace{-0.2in}
\end{table}%
 \noindent\textbf{Experiment Details}.
We follow the procedures in LoRA \citep{hu2021lora} to implement language generation, including (1) fine-tuning the language model, (2) generating outputs for text data using beam search, (3) decoding the outputs, and (4) evaluating the generated outputs.  
The NLG experiments are run across 8 clients (NVIDIA A100 GPU), where each client fine-tunes the adapter on the data of specific domains (WebNLG) or restaurants (E2E NLG Challenge), and then generates individual outputs for the client test data during the evaluation phase. We use metrics including BLEU, NIST, METEOR (MET), TER, ROUGE-L, CIDEr to measure the quality of generated texts.

The total communication rounds $R$ are set to 200 for WebNLG  and 300 for E2E NLG Challenge, and communication interval is fixed as $I=10$ for both datasets. The optimizer setting follows the previous Section \ref{sec:robert_exp_detail}. The batch size $\mathcal{B}=4$, and beam search width $bw=10$ are kept for all the baselines.  We tune the best step size for each baseline, and the details are summarized in Table \ref{tab:web_e2e_lr} and Table \ref{tab:web_e2e_lr_GPT2_XL} in Appendix \ref{sec:hy_nlg}. In addition, the rank initialization for each algorithm and the number of trainable parameters are summarized in Table \ref{tab:gpt2_hyperparam} in Appendix \ref{sec:hy_nlg}. The test results for GPT-2 medium are presented in Tables \ref{tab:result_webnlg} and 
\ref{tab:result_e2e} in Appendix \ref{sec:result_e2e}, and the results of GPT2-XL are presented in Table \ref{tab:result_webnlg_GPT2_XL}. We can see that PF2LoRA achieves the best performance in almost all metrics compared to other federated fine-tuning baselines. For example, PF2LoRA on GPT-2 medium achieves $1.25\%$ higher BLEU score than HETLoRA on WebNLG and $3.85\%$ higher BLUE score than HETLoRA on E2E NLG Challenge.

\subsection{Ablation Studies}
\label{sec:ablation}
%\vspace*{-0.15in}
We execute the ablation studies to explore  (1) the performance comparison when other baselines have more trainable parameters than ours. (2) the impact of data heterogeneity on PF2LoRA and baselines. (3) the importance of bilevel optimization in our framework. Due to the space limitation, the details about (2) and (3) have been deferred to the Appendix. Refer to Appendix \ref{sec:het_level} for the impact of data heterogeneity and Appendix \ref{sec:ablation_BO} for the importance of bilevel optimization.

\noindent\textbf{Baselines with More Trainable Parameters.} \label{sec:basline_more_params}
The lightweight client-specific adapters introduce additional trainable parameters. For fair comparison with other baselines, we consider to increase the number of trainable parameters in other baselines. Specifically, we increase the initial rank $r_k$ (from $8$ to $12$) for baselines HOMLoRA and Per-FedAvg-LoRA in the text classification experiments. Note that HETLoRA has different rank initialization $r_{min}\leq r_k \leq r_{max}$ for different client $k$, so we count the average trainable parameters of the clients. we can also control the number of trainable parameters by specifying $r_{min}$ and $r_{max}$. We specify  $r_{min}=5, r_{max}=12$ in CoLA dataset and $r_{min}=8, r_{max}=12$ in other four text classification datasets. The number of trainable parameters of each baseline and the corresponding test score in each dataset are summarized in Table \ref{tab:more_prams_result}. Even if other algorithms have more trainable parameters than our method, PF2LoRA still demonstrates the best performance. PF2LoRA, with negligible additional trainable parameters, significantly improves the performance in personalized federated learning.     

%\vspace{-0.1in}
\section{Theoretical Justification}
In this section, we provide the theoretical justification for the Algorithm~\ref{alg:PF2LORA} in an simplified scenario: we consider the single machine case ($M=1$) and assume we have access to the deterministic gradient oracle. In this case the algorithm reduces the following formulation:
%\vspace*{-0.1in}
\begin{equation}\label{eq:bilevel_obj_singlemachine}
    \begin{aligned}
        &\min_{x} \Phi(x)\coloneq f(x,y^*(x)),  & (\text{UL})\\
        & \text{s.t.}, \ y^*(x) \in \arg\min_{y} f(x, y),  &(\text{LL}),
    \end{aligned}
\end{equation}
% where $f(x,y)=\mathbb{E}_{\xi\sim\mathcal{D}}\left[F(x, y^*(x);\xi)\right]$, and $\mathcal{D}$ is the underlying data distribution. 
The update of Algorithm~\ref{alg:PF2LORA} in the single machine case with deterministic gradient reduces to the following update rule:
% \vspace*{-0.1in}
\begin{equation}
\label{update:singlemachine}
\begin{aligned}
    y^{t+1} &= y^t  - \alpha\nabla_y F(x^t,y^t)\\
    x^{t+1} &= x^t - \eta [\nabla_x F(x^{t}, y^{t+1}) 
    + \alpha\nabla_{xy}F(x^{t}, y^{t})\nabla_yF(x^t, y^{t+1})].
\end{aligned}
 \end{equation}

We will establish the convergence of the update rule~\eqref{update:singlemachine} under the following assumptions.

\begin{assumption}
\label{ass:bilevel}
(i) $f$ are bounded below, $\Phi(x_0)-\min_{x}\Phi(x)\leq \Delta$; (ii) $f$ is $\mu$-strongly convex in terms of $y$ for given $x$ ; (iii) $f$ is continuously differentiable and $L_{f,1}$-smooth jointly in $(x,y)$; (iv) $f$ is twicely differentiable and $\nabla^2 f$ is $L_{f,2}$-Lipschitz jointly in $(x,y)$.
%; (v) We access the gradients of the objective function $f$ via unbiased estimator $\nabla f(x,y,\xi)$ where $\mathbb{E}\left[\nabla f(x,y,\xi)\right]=f(x,y)$ and $\mathbb{E} \left[\|\nabla f(x,y,\xi)-\nabla f(x,y)\|^2\right]\leq \sigma^2$ for any $x,y$; (vi) The properties in (i) (ii) (iii) (iv) hold for individual function $F(x,y;\xi)$.
% \begin{enumerate}[(i)]
%     %\item $f_k(x,y)=f_k(x-\alpha\nabla f_i(x,y))$
%     \item $f$ are bounded below, $\Phi(x_0)-\min_{x}\Phi(x)\leq \Delta$. %$f$ is $G$-Lipschitz.
%     \item $f$ is $\mu$-strongly convex in terms of $y$ for given $x$ .
%     \item $f$ is continuously differentiable and $L_{f,1}$-smooth jointly in $(x,y)$.
%     \item $f$ is twicely differentiable and $\nabla^2 f$ is $L_{f,2}$-Lipschitz jointly in $(x,y)$.
%     \item We access the gradients of the objective function $f$ via unbiased estimator $\nabla f(x,y,\xi)$ where $\mathbb{E}\left[\nabla f(x,y,\xi)\right]=f(x,y)$ and $\mathbb{E} \left[\|\nabla f(x,y,\xi)-\nabla f(x,y)\|^2\right]\leq \sigma^2$ for any $x,y$.
%     \item The properties in (i) (ii) (iii) (iv) hold for individual function $F(x,y;\xi)$.

% \end{enumerate}
    \end{assumption}
\textbf{Remark}: These assumptions are standard in the bilevel optimization literature~\citep{kwon2023fully,ji2021bilevel}.

\begin{theorem}[Convergence Guarantees]\label{mainthm1}
Suppose Assumption~\ref{ass:bilevel} holds. Define the smoothness parameter $L_\Phi=L_{f,1}+\frac{L_{f,1}^2}{\mu}$, and choose $\alpha=\frac{1}{4L_{f,1}}, \eta=\min\left(\frac{\mu^2}{5L_{f,1}^3\sqrt{(\frac{4L_{f,1}}{\mu}-\frac{\mu}{4L_{f,1}})}}, \frac{1}{8L_\Phi}, \sqrt{\frac{1}{16N}}, \sqrt[3]{\frac{1}{81NL_\Phi}}\right)$, and $N=\frac{25L_{f,1}^4(\frac{4L_{f,1}}{\mu}+1)}{16\mu^2}$. Then, we have $\frac{1}{T}\sum_{t=0}^{T-1}\|\nabla\Phi(x^t)\|^2 \leq
    O\left(1/T\right)$, where $T$ is the total number of iterations.
% \begin{itemize}
% \item In deterministic case ($\sigma=0$), we have
% \begin{equation}
%     \frac{1}{T}\sum_{t=0}^{T-1}\|\nabla\Phi(x^t)\|^2 \leq
%     O\left(\frac{\Delta}{\eta T}+\epsilon\right).
% \end{equation}
% \item In stochastic case ($\sigma>0$), choosing batch size $|\gB|=T$, we have
% \begin{equation}
%     \frac{1}{T}\sum_{t=0}^{T-1}\E\|\nabla\Phi(x^t)\|^2 \leq
%     O(\frac{\Delta}{\eta T}+\frac{\sigma^2}{|\gB|})= O\left(\frac{\Delta}{\eta T}+\frac{\sigma^2}{T}\right).
% \end{equation}
% \end{itemize}

\end{theorem}
\textbf{Remark}: Theorem~\ref{mainthm1} provides a convergence guarantee with $O(1/T)$ convergence rate for the squared gradient norm. It means that it requires $O(1/\epsilon^2)$ gradient or Hessian-vector product evaluations for finding an $\epsilon$-stationary point (i.e., finding a $x$ such that $\|\nabla \Phi(x)\|\leq  \epsilon$). This complexity matches the convergence rate of gradient descent for smooth nonconvex function. In addition, compared with existing double-loop bilevel optimization algorithms such as~\cite{ji2021bilevel}, our update rule~\eqref{update:singlemachine} is an single-loop bilevel algorithm and hence is easy to implement in practice.

\vspace*{-0.1in}
\section{Conclusion}
%\vspace*{-0.05in}
%\vspace{-0.15in}
In this paper, we presented PF2LoRA, a novel personalized federated fine-tuning algorithm for heterogeneous data based on a two-level LoRA framework, where the first level aims to learns a common adapter for all the clients and the second level fosters individual client personalization. Our approach achieves automatic rank learning and addresses the limitations of existing methods, such as data-independent rank initialization and excessive hyperparameter tuning. Through comprehensive experiments on NLU and NLG tasks, PF2LoRA demonstrated significant performance improvements over state-of-the-art baselines, with negligible additional memory overhead.

\section*{Impact Statement}

This paper presents work whose goal is to advance the field of 
Machine Learning. There are many potential societal consequences 
of our work, none which we feel must be specifically highlighted here.

\nocite{langley00}

\bibliography{ref}

\begin{thebibliography}{44}
\providecommand{\natexlab}[1]{#1}
\providecommand{\url}[1]{\texttt{#1}}
\expandafter\ifx\csname urlstyle\endcsname\relax
  \providecommand{\doi}[1]{doi: #1}\else
  \providecommand{\doi}{doi: \begingroup \urlstyle{rm}\Url}\fi

\bibitem[Babakniya et~al.(2023)Babakniya, Elkordy, Ezzeldin, Liu, Song,
  El-Khamy, and Avestimehr]{babakniya2023slora}
Babakniya, S., Elkordy, A.~R., Ezzeldin, Y.~H., Liu, Q., Song, K.-B., El-Khamy,
  M., and Avestimehr, S.
\newblock Slora: Federated parameter efficient fine-tuning of language models.
\newblock \emph{arXiv preprint arXiv:2308.06522}, 2023.

\bibitem[Bonawitz et~al.(2019)Bonawitz, Eichner, Grieskamp, Huba, Ingerman,
  Ivanov, Kiddon, Kone{\v{c}}n{\`y}, Mazzocchi, McMahan,
  et~al.]{bonawitz2019towards}
Bonawitz, K., Eichner, H., Grieskamp, W., Huba, D., Ingerman, A., Ivanov, V.,
  Kiddon, C., Kone{\v{c}}n{\`y}, J., Mazzocchi, S., McMahan, B., et~al.
\newblock Towards federated learning at scale: System design.
\newblock \emph{Proceedings of machine learning and systems}, 1:\penalty0
  374--388, 2019.

\bibitem[Brisimi et~al.(2018)Brisimi, Chen, Mela, Olshevsky, Paschalidis, and
  Shi]{brisimi2018federated}
Brisimi, T.~S., Chen, R., Mela, T., Olshevsky, A., Paschalidis, I.~C., and Shi,
  W.
\newblock Federated learning of predictive models from federated electronic
  health records.
\newblock \emph{International journal of medical informatics}, 112:\penalty0
  59--67, 2018.

\bibitem[Cho et~al.(2024)Cho, Liu, Xu, Fahrezi, and
  Joshi]{cho2024heterogeneous}
Cho, Y.~J., Liu, L., Xu, Z., Fahrezi, A., and Joshi, G.
\newblock Heterogeneous low-rank approximation for federated fine-tuning of
  on-device foundation models.
\newblock \emph{arXiv preprint arXiv:2401.06432}, 2024.

\bibitem[Devlin et~al.(2018)Devlin, Chang, Lee, and Toutanova]{devlin2018bert}
Devlin, J., Chang, M.-W., Lee, K., and Toutanova, K.
\newblock Bert: Pre-training of deep bidirectional transformers for language
  understanding.
\newblock \emph{arXiv preprint arXiv:1810.04805}, 2018.

\bibitem[Fallah et~al.(2020)Fallah, Mokhtari, and
  Ozdaglar]{fallah2020personalized}
Fallah, A., Mokhtari, A., and Ozdaglar, A.
\newblock Personalized federated learning: A meta-learning approach.
\newblock \emph{arXiv preprint arXiv:2002.07948}, 2020.

\bibitem[Finn et~al.(2017)Finn, Abbeel, and Levine]{finn2017model}
Finn, C., Abbeel, P., and Levine, S.
\newblock Model-agnostic meta-learning for fast adaptation of deep networks.
\newblock In \emph{International conference on machine learning}, pp.\
  1126--1135. PMLR, 2017.

\bibitem[Gardent et~al.(2017)Gardent, Shimorina, Narayan, and
  Beltrachini]{gardent2017webnlg}
Gardent, C., Shimorina, A., Narayan, S., and Beltrachini, L.~P.
\newblock The webnlg challenge: Generating text from rdf data.
\newblock In \emph{Proceedings of the 10th international conference on natural
  language generation}, pp.\  124--133, 2017.

\bibitem[Geyer et~al.(2017)Geyer, Klein, and Nabi]{geyer2017differentially}
Geyer, R.~C., Klein, T., and Nabi, M.
\newblock Differentially private federated learning: A client level
  perspective.
\newblock \emph{arXiv preprint arXiv:1712.07557}, 2017.

\bibitem[Ghadimi \& Wang(2018)Ghadimi and Wang]{ghadimi2018approximation}
Ghadimi, S. and Wang, M.
\newblock Approximation methods for bilevel programming.
\newblock \emph{arXiv preprint arXiv:1802.02246}, 2018.

\bibitem[Han et~al.(2015)Han, Mao, and Dally]{han2015deep}
Han, S., Mao, H., and Dally, W.~J.
\newblock Deep compression: Compressing deep neural networks with pruning,
  trained quantization and huffman coding.
\newblock \emph{arXiv preprint arXiv:1510.00149}, 2015.

\bibitem[He et~al.(2021{\natexlab{a}})He, Zhou, Ma, Berg-Kirkpatrick, and
  Neubig]{he2021towards}
He, J., Zhou, C., Ma, X., Berg-Kirkpatrick, T., and Neubig, G.
\newblock Towards a unified view of parameter-efficient transfer learning.
\newblock \emph{arXiv preprint arXiv:2110.04366}, 2021{\natexlab{a}}.

\bibitem[He et~al.(2021{\natexlab{b}})He, Gao, and Chen]{he2021debertav3}
He, P., Gao, J., and Chen, W.
\newblock Debertav3: Improving deberta using electra-style pre-training with
  gradient-disentangled embedding sharing.
\newblock \emph{arXiv preprint arXiv:2111.09543}, 2021{\natexlab{b}}.

\bibitem[Houlsby et~al.(2019)Houlsby, Giurgiu, Jastrzebski, Morrone,
  De~Laroussilhe, Gesmundo, Attariyan, and Gelly]{houlsby2019parameter}
Houlsby, N., Giurgiu, A., Jastrzebski, S., Morrone, B., De~Laroussilhe, Q.,
  Gesmundo, A., Attariyan, M., and Gelly, S.
\newblock Parameter-efficient transfer learning for nlp.
\newblock In \emph{International conference on machine learning}, pp.\
  2790--2799. PMLR, 2019.

\bibitem[Hu et~al.(2021)Hu, Shen, Wallis, Allen-Zhu, Li, Wang, Wang, and
  Chen]{hu2021lora}
Hu, E.~J., Shen, Y., Wallis, P., Allen-Zhu, Z., Li, Y., Wang, S., Wang, L., and
  Chen, W.
\newblock Lora: Low-rank adaptation of large language models.
\newblock \emph{arXiv preprint arXiv:2106.09685}, 2021.

\bibitem[Ji et~al.(2021)Ji, Yang, and Liang]{ji2021bilevel}
Ji, K., Yang, J., and Liang, Y.
\newblock Bilevel optimization: Convergence analysis and enhanced design.
\newblock In \emph{International conference on machine learning}, pp.\
  4882--4892. PMLR, 2021.

\bibitem[Kairouz et~al.(2021)Kairouz, McMahan, Avent, Bellet, Bennis, Bhagoji,
  Bonawitz, Charles, Cormode, Cummings, et~al.]{kairouz2021advances}
Kairouz, P., McMahan, H.~B., Avent, B., Bellet, A., Bennis, M., Bhagoji, A.~N.,
  Bonawitz, K., Charles, Z., Cormode, G., Cummings, R., et~al.
\newblock Advances and open problems in federated learning.
\newblock \emph{Foundations and trends{\textregistered} in machine learning},
  14\penalty0 (1--2):\penalty0 1--210, 2021.

\bibitem[Karimireddy et~al.(2020)Karimireddy, Kale, Mohri, Reddi, Stich, and
  Suresh]{karimireddy2020scaffold}
Karimireddy, S.~P., Kale, S., Mohri, M., Reddi, S., Stich, S., and Suresh,
  A.~T.
\newblock Scaffold: Stochastic controlled averaging for federated learning.
\newblock In \emph{International conference on machine learning}, pp.\
  5132--5143. PMLR, 2020.

\bibitem[Kwon et~al.(2023)Kwon, Kwon, Wright, and Nowak]{kwon2023fully}
Kwon, J., Kwon, D., Wright, S., and Nowak, R.~D.
\newblock A fully first-order method for stochastic bilevel optimization.
\newblock In \emph{International Conference on Machine Learning}, pp.\
  18083--18113. PMLR, 2023.

\bibitem[Lawton et~al.(2023)Lawton, Kumar, Thattai, Galstyan, and
  Steeg]{lawton2023neural}
Lawton, N., Kumar, A., Thattai, G., Galstyan, A., and Steeg, G.~V.
\newblock Neural architecture search for parameter-efficient fine-tuning of
  large pre-trained language models.
\newblock \emph{arXiv preprint arXiv:2305.16597}, 2023.

\bibitem[Lee et~al.(2019)Lee, Tang, and Lin]{lee2019would}
Lee, J., Tang, R., and Lin, J.
\newblock What would elsa do? freezing layers during transformer fine-tuning.
\newblock \emph{arXiv preprint arXiv:1911.03090}, 2019.

\bibitem[Lester et~al.(2021)Lester, Al-Rfou, and Constant]{lester2021power}
Lester, B., Al-Rfou, R., and Constant, N.
\newblock The power of scale for parameter-efficient prompt tuning.
\newblock \emph{arXiv preprint arXiv:2104.08691}, 2021.

\bibitem[Li \& Liang(2021)Li and Liang]{li2021prefix}
Li, X.~L. and Liang, P.
\newblock Prefix-tuning: Optimizing continuous prompts for generation.
\newblock \emph{arXiv preprint arXiv:2101.00190}, 2021.

\bibitem[Liu et~al.(2019)Liu, Ott, Goyal, Du, Joshi, Chen, Levy, Lewis,
  Zettlemoyer, and Stoyanov]{liu2019roberta}
Liu, Y., Ott, M., Goyal, N., Du, J., Joshi, M., Chen, D., Levy, O., Lewis, M.,
  Zettlemoyer, L., and Stoyanov, V.
\newblock Roberta: A robustly optimized bert pretraining approach.
\newblock \emph{arXiv preprint arXiv:1907.11692}, 2019.

\bibitem[McMahan et~al.(2017{\natexlab{a}})McMahan, Moore, Ramage, Hampson, and
  y~Arcas]{mcmahan2017communication}
McMahan, B., Moore, E., Ramage, D., Hampson, S., and y~Arcas, B.~A.
\newblock Communication-efficient learning of deep networks from decentralized
  data.
\newblock In \emph{Artificial intelligence and statistics}, pp.\  1273--1282.
  PMLR, 2017{\natexlab{a}}.

\bibitem[McMahan et~al.(2017{\natexlab{b}})McMahan, Ramage, Talwar, and
  Zhang]{mcmahan2017learning}
McMahan, H.~B., Ramage, D., Talwar, K., and Zhang, L.
\newblock Learning differentially private recurrent language models.
\newblock \emph{arXiv preprint arXiv:1710.06963}, 2017{\natexlab{b}}.

\bibitem[Novikova et~al.(2017)Novikova, Du{\v{s}}ek, and
  Rieser]{novikova2017e2e}
Novikova, J., Du{\v{s}}ek, O., and Rieser, V.
\newblock The e2e dataset: New challenges for end-to-end generation.
\newblock \emph{arXiv preprint arXiv:1706.09254}, 2017.

\bibitem[Radford et~al.(2018)Radford, Narasimhan, Salimans, Sutskever,
  et~al.]{radford2018improving}
Radford, A., Narasimhan, K., Salimans, T., Sutskever, I., et~al.
\newblock Improving language understanding by generative pre-training.
\newblock 2018.

\bibitem[Radford et~al.(2019)Radford, Wu, Child, Luan, Amodei, Sutskever,
  et~al.]{radford2019language}
Radford, A., Wu, J., Child, R., Luan, D., Amodei, D., Sutskever, I., et~al.
\newblock Language models are unsupervised multitask learners.
\newblock \emph{OpenAI blog}, 1\penalty0 (8):\penalty0 9, 2019.

\bibitem[Rajpurkar et~al.(2016)Rajpurkar, Zhang, Lopyrev, and
  Liang]{rajpurkar2016squad}
Rajpurkar, P., Zhang, J., Lopyrev, K., and Liang, P.
\newblock Squad: 100,000+ questions for machine comprehension of text.
\newblock \emph{arXiv preprint arXiv:1606.05250}, 2016.

\bibitem[Rajpurkar et~al.(2018)Rajpurkar, Jia, and Liang]{rajpurkar2018know}
Rajpurkar, P., Jia, R., and Liang, P.
\newblock Know what you don't know: Unanswerable questions for squad.
\newblock \emph{arXiv preprint arXiv:1806.03822}, 2018.

\bibitem[Reinsel \& Velu(1998)Reinsel and Velu]{reinsel1998multivariate}
Reinsel, G.~C. and Velu, R.~P.
\newblock \emph{Multivariate reduced-rank regression}.
\newblock Springer, 1998.

\bibitem[Sheller et~al.(2020)Sheller, Edwards, Reina, Martin, Pati, Kotrotsou,
  Milchenko, Xu, Marcus, Colen, et~al.]{sheller2020federated}
Sheller, M.~J., Edwards, B., Reina, G.~A., Martin, J., Pati, S., Kotrotsou, A.,
  Milchenko, M., Xu, W., Marcus, D., Colen, R.~R., et~al.
\newblock Federated learning in medicine: facilitating multi-institutional
  collaborations without sharing patient data.
\newblock \emph{Scientific reports}, 10\penalty0 (1):\penalty0 12598, 2020.

\bibitem[Sung et~al.(2021)Sung, Nair, and Raffel]{sung2021training}
Sung, Y.-L., Nair, V., and Raffel, C.~A.
\newblock Training neural networks with fixed sparse masks.
\newblock \emph{Advances in Neural Information Processing Systems},
  34:\penalty0 24193--24205, 2021.

\bibitem[Wang et~al.(2018)Wang, Singh, Michael, Hill, Levy, and
  Bowman]{wang2018glue}
Wang, A., Singh, A., Michael, J., Hill, F., Levy, O., and Bowman, S.~R.
\newblock Glue: A multi-task benchmark and analysis platform for natural
  language understanding.
\newblock \emph{arXiv preprint arXiv:1804.07461}, 2018.

\bibitem[Wang et~al.(2023)Wang, Zhang, Cao, Li, McMahan, Oh, Xu, and
  Zaheer]{wang2023can}
Wang, B., Zhang, Y.~J., Cao, Y., Li, B., McMahan, H.~B., Oh, S., Xu, Z., and
  Zaheer, M.
\newblock Can public large language models help private cross-device federated
  learning?
\newblock \emph{arXiv preprint arXiv:2305.12132}, 2023.

\bibitem[Wei et~al.(2021)Wei, Xie, and Ma]{wei2021pretrained}
Wei, C., Xie, S.~M., and Ma, T.
\newblock Why do pretrained language models help in downstream tasks? an
  analysis of head and prompt tuning.
\newblock \emph{Advances in Neural Information Processing Systems},
  34:\penalty0 16158--16170, 2021.

\bibitem[Xu et~al.(2021)Xu, Luo, Zhang, Tan, Chang, Huang, and
  Huang]{xu2021raise}
Xu, R., Luo, F., Zhang, Z., Tan, C., Chang, B., Huang, S., and Huang, F.
\newblock Raise a child in large language model: Towards effective and
  generalizable fine-tuning.
\newblock \emph{arXiv preprint arXiv:2109.05687}, 2021.

\bibitem[Yang et~al.(2019)Yang, Liu, Chen, and Tong]{yang2019federated}
Yang, Q., Liu, Y., Chen, T., and Tong, Y.
\newblock Federated machine learning: Concept and applications.
\newblock \emph{ACM Transactions on Intelligent Systems and Technology (TIST)},
  10\penalty0 (2):\penalty0 1--19, 2019.

\bibitem[Zaken et~al.(2021)Zaken, Ravfogel, and Goldberg]{zaken2021bitfit}
Zaken, E.~B., Ravfogel, S., and Goldberg, Y.
\newblock Bitfit: Simple parameter-efficient fine-tuning for transformer-based
  masked language-models.
\newblock \emph{arXiv preprint arXiv:2106.10199}, 2021.

\bibitem[Zhang et~al.(2024)Zhang, Vahidian, Kuo, Li, Zhang, Yu, Wang, and
  Chen]{zhang2024towards}
Zhang, J., Vahidian, S., Kuo, M., Li, C., Zhang, R., Yu, T., Wang, G., and
  Chen, Y.
\newblock Towards building the federatedgpt: Federated instruction tuning.
\newblock In \emph{ICASSP 2024-2024 IEEE International Conference on Acoustics,
  Speech and Signal Processing (ICASSP)}, pp.\  6915--6919. IEEE, 2024.

\bibitem[Zhang et~al.(2023{\natexlab{a}})Zhang, Chen, Bukharin, He, Cheng,
  Chen, and Zhao]{zhang2023adaptive}
Zhang, Q., Chen, M., Bukharin, A., He, P., Cheng, Y., Chen, W., and Zhao, T.
\newblock Adaptive budget allocation for parameter-efficient fine-tuning.
\newblock In \emph{International Conference on Learning Representations}.
  Openreview, 2023{\natexlab{a}}.

\bibitem[Zhang et~al.(2023{\natexlab{b}})Zhang, Yang, Dai, Wang, Yu, Qu, and
  Xu]{zhang2023fedpetuning}
Zhang, Z., Yang, Y., Dai, Y., Wang, Q., Yu, Y., Qu, L., and Xu, Z.
\newblock Fedpetuning: When federated learning meets the parameter-efficient
  tuning methods of pre-trained language models.
\newblock In \emph{Annual Meeting of the Association of Computational
  Linguistics 2023}, pp.\  9963--9977. Association for Computational
  Linguistics (ACL), 2023{\natexlab{b}}.

\bibitem[Zhao et~al.(2020)Zhao, Lin, Mi, Jaggi, and
  Sch{\"u}tze]{zhao2020masking}
Zhao, M., Lin, T., Mi, F., Jaggi, M., and Sch{\"u}tze, H.
\newblock Masking as an efficient alternative to finetuning for pretrained
  language models.
\newblock \emph{arXiv preprint arXiv:2004.12406}, 2020.

\end{thebibliography}
\bibliographystyle{icml2025}

%%%%%%%%%%%%%%%%%%%%%%%%%%%%%%%%%%%%%%%%%%%%%%%%%%%%%%%%%%%%%%%%%%%%%%%%%%%%%%%
%%%%%%%%%%%%%%%%%%%%%%%%%%%%%%%%%%%%%%%%%%%%%%%%%%%%%%%%%%%%%%%%%%%%%%%%%%%%%%%
% APPENDIX
%%%%%%%%%%%%%%%%%%%%%%%%%%%%%%%%%%%%%%%%%%%%%%%%%%%%%%%%%%%%%%%%%%%%%%%%%%%%%%%
%%%%%%%%%%%%%%%%%%%%%%%%%%%%%%%%%%%%%%%%%%%%%%%%%%%%%%%%%%%%%%%%%%%%%%%%%%%%%%%
\newpage
\appendix
\onecolumn
\section{Details of Per-FedAvg-LoRA}
\label{app:per-fedavg}
\noindent\textbf{Per-FedAvg-LoRA}. Per-FedAvg-LoRA is built upon a well-known personalized federated learning approach called Per-FedAvg~\citep{fallah2020personalized}, with the trainable model parameters being low rank matrices such as in LoRA. Per-FedAvg is a typical personalized federated learning algorithm, which incorporates  Model-Agnostic Meta-Learning (MAML) \citep{finn2017model} to FedAvg algorithm~\citep{mcmahan2017communication} to enable models quickly adapting to heterogeneous data. When it is applied to low rank adaptation, we can get a new variant, namely Per-FedAvg-LoRA.
The goal of Per-FedAvg-LoRA is to find a common adapter $x$ which can perform well after it is updated by one-step gradient descent on each client. In particular, Per-FedAvg-LoRA is trying to solve the following formulation using the FedAvg algorithm:
\vspace*{-0.15in}
\begin{equation}
    \min_x \frac{1}{M}\sum_{k=1}^{M}f_k(x- \alpha\nabla f_k(x)),
\end{equation}
where $\alpha>0$ is the step size. Note that Per-FedAvg-LoRA uses adapter matrices with homogeneous rank across all the clients.

\section{PyTorch-style Pseudocode for PF2LoRA} \label{sec:pseudocode}
In this section, we show the PyTorch-style pseudocode for PF2LoRA. Our two-level low rank adapter framework can be derived by slightly modifying the LoRA module and integrating it into federating learning. When creating low rank adapters, we need to initialize two types of adapters, i.e., common adapters and the client adapters. The initial rank dimension for the common adapter   is typically set to $r$, while for the client adapter, it is set to $\frac{r}{2}$. In addition, we require two different optimizers to update the common and client adapters. The common adapter is updated using AdamW, and the client adapter is updated using SGD.  It's important to note that hypergradient calculation is necessary when updating the common adapter. Besides, our framework can be easily plugged into multiple language models, such as RoBERTa, DeBERTa and GPT-2, and others.    

\label{sec:appendix}
\begin{algorithm*}[h]
\caption{PyTorch-style Pseudocode for PF2LoRA }\label{alg:pytorch_PF2LoRA}
\begin{algorithmic}[2]
\vspace{-0.1in}
\begin{lstlisting}[language=Python,  emph={update_client_adapter, update_common_adapter, initialize_client_adapter, SGD}, 
  emphstyle=\color{cyan}]
# model_name: the name of pretrained model
# lr_in, lr_out: the learning rate for client and common adapter
# T: the total number of communication rounds, I: communication interval
# r: low rank parameter
# train_dataloader

import torch.distributed as dist
dist.init_process_group() 
target_modules = ["query", "value"]
pretrained_model = LLM_Model.from_pretrained(model_name)
model = get_peft_model(pretrained_model, target_modules, r)
optimizer_outer = AdamW(model.common_adpter.parameters(), lr_in)
optimizer_inner = SGD(model.client_adpter.parameters(), lr_out)

step = 0 
for epoch_idx in range(total_epochs)
    for data_batch in train_dataloader:
        inner_batch, outer_batch = data_batch
        update_client_adapter(model, inner_batch, optimizer_inner)
        update_common_adapter(model, outer_batch, optimizer_outer)
        if step % I == 0:
            dist.reduce(model.common_adapter.parameters(), dst=0,  
     op=self.dist.ReduceOp.SUM)
            average(model.common_adpter.parameters())
            dist.broadcast(model.common_adapter.parameters(), src=0)
        step += 1
#         
def get_peft_model(pretrained_model, target_modules, r)
    for module_name, _ in pretrained_model.named_modules():
        if module_name in target_modules:
             target_module= pretrained_model.get_submodule(module_name)
             create_and_replace(target_module, r)

def create_and_replace(target_module, r)
    if isinstance(target_module, Linear):
        target_module.initialize_common_adapter(r)
        target_module.initialize_client_adapter(r/2)
        target_module.set_trainable_params()
        
\end{lstlisting}
\vspace{-0.2in}
\STATE
\end{algorithmic}
\end{algorithm*}

\section{Experiment Setup} \label{sec:hp_setting}

\subsection{RoBERTa on Text Classification} \label{sec:hy_txt_cl}
We use grid search to find the best learning rate for each algorithm in the range of $\{1.0\times 10^{-4}, 5.0\times 10^{-4}, 1.0\times 10^{-3}, 2.0\times 10^{-3}, 5.0\times 10^{-3} \}$. For algorithm Per-FedAvg-LoRA, we search the optimal learning rate for one-step update and the common adapter update, respectively. For PF2LoRA, we also search for the best learning rate for the client-specific adapter update and the common adapter update. The selected learning rates for each algorithm are listed in Table \ref{tab:lr_roberta}, where we use slash to separate two learning rates for Per-FedAvg-LoRA and PF2LoRA, with the former learning rate being for the common adapter. For HETLoRA, we fix the sparsity parameter $\gamma=0.99$ across all the datasets and set the penalty factor $\lambda=1.0\times 10 ^{-3}$ for CoLA dataset, and $\lambda=5.0\times 10 ^{-3}$ for MNLI, SST-2, QQP, and QNLI. The rank initialization and the number of trainable parameters are summarized in Table \ref{tab:parameters_roberta}. The details of the text classification datasets are summarized in Table \ref{tab:glue_benchmark}.

\begin{table*}[htbp]
  \centering
  \caption{Learning rate setting for RoBERETa model on GLUE benchmark. We use slash to separate two learning rates for Per-FedAvg-LoRA and PF2LoRA. For Per-FedAvg-LoRA, the former one is the learning rate for the common adapter, the latter one is the learning rate for one-step SGD. For PF2LoRA,  the former one is the learning rate for the common adapter, the latter one is the learning rate for the client-specific adapter.}
    \scalebox{0.75}{
    \begin{tabular}{l|ccccc}
    \Xhline{1.2pt}
    Method  & CoLA  & MNLI  & SST-2  & QQP   & QNLI \\
    \hline
    Centralized LoRA & $1.0\times 10^{-3}$ & $1.0\times 10^{-3} $ & $1.0\times 10^{-3}$ & $1.0\times 10^{-3}$ &$1.0\times 10^{-3}$\\

    HOMLoRA & $1.0\times 10^{-3}$ & $1.0\times 10^{-3}$ & $2.0\times 10^{-3}$ & $1.0\times 10^{-3}$ & $1.0\times 10^{-3}$\\

    Per-FedAvg-LoRA & $2.0\times 10^{-3} / 1.0\times 10^{-2}$ & $1.0\times 10^{-3} / 1.0\times 10^{-3}$ & $1.0\times 10^{-3} / 1.0\times 10^{-3}$ & $1.0\times 10^{-3} / 1.0\times 10^{-3}$ & $2.0\times 10^{-3} / 1.0\times 10^{-3}$ \\

    HETLoRA & $5.0\times 10^{-3}$ & $2.0\times 10^{-3}$ & $2.0\times 10^{-3}$ & $2.0\times 10^{-3}$ & $2.0\times 10^{-3}$ \\

    PF2LoRA  & $2.0\times 10^{-3}/1.0\times 10^{-4}$ & $1.0\times 10^{-3} / 1.0\times 10^{-3}$ & $1.0\times 10^{-3} / 1.0\times 10^{-3}$ & $1.0\times 10^{-3} / 1.0\times 10^{-3}$ & $1.0\times 10^{-3} / 1.0\times 10^{-3}$ \\
    \Xhline{1.2pt}
    \end{tabular}}%
  \label{tab:lr_roberta}%
\end{table*}%

\begin{table*}[htbp]
  \centering
  \caption{Trainable parameters of RoBERTa-base/large.}
    % \scalebox{0.75}{
    \begin{tabular}{l|cc}
    \Xhline{1.2pt}
    \multicolumn{1}{l|}{\multirow{2}[2]{*}{Method}} & \multicolumn{1}{c}{\# Trainable Parameters} \\
          & (base/large)  \\
    \hline
    HOMLoRA ($r_k=8$) & 0.30M/0.79M  \\
    Per-FedAvg-LoRA ($r_k=8$) & 0.30M/0.79M \\
    HETLoRAS ($r_{max}=12, r_{min}=8$)   & 0.35M/0.94M       \\
    PF2LoRA ($r_k=8, \tilde{r}_{k}=2$) & 0.37M/0.99M  \\
    \Xhline{1.2pt}
    \end{tabular}%
  \label{tab:parameters_roberta}%
\end{table*}%

\begin{table}[htbp]
  \centering
  \caption{The summary of GLUE benchmark.}
    % \scalebox{1}{
    \begin{tabular}{l|cccc}
    \Xhline{1.2pt}
    Corpus & \# Train & \# Test & \# Lable & Metrics \\
    \hline
    CoLA  & 8.5k  & 1k    & 2     & Matthew's correlation \\
    MNLI  & 393k  & 20k   & 3     & Accuracy \\
    SST-2 & 67k   & 1.8k  & 2     & Accuracy \\
    QQP   & 364k  & 391k  & 2     & Accuracy \\
    QNLI  & 108k  & 5.7k  & 2     & Accuracy \\
    \Xhline{1.2pt}
    \end{tabular}%
  \label{tab:glue_benchmark}%
\end{table}%

\subsection{DeBERTa on Question Answering} \label{sec:hy_qa}
% batch size 16, gpus=4
% HETLoRA: $=5.0\times 10^{-3}$, $\gamma=0.99$
We search for the optimal learning rate from the range of $\{1.0\times 10^{-4}, 5.0\times 10^{-4}, 1.0\times 10^{-3}, 2.0\times 10^{-3}, 5.0\times 10^{-3}, 1.0\times 10^{-2} \}$ for each algorithm on SQuAD v1 and v2 dataset. Refer to Table \ref{tab:squadv1_lr} for detailed learning rate settings. The rank initialization and the number of trainable parameters for different algorithms are presented in Table \ref{tab:deberta_hyperparam}. For PF2LoRA, we initialize the rank of client-specific adapter $\tilde{r}_k=\frac{r_k}{2}=4$, and we set the best value of $r_{min}=6, r_{max}=14$ for HETLoRA.  In addition, HETLoRA uses the sparsity parameter $\gamma=0.99$ and the penalty factor $\lambda=5.0\times 10 ^{-3}$ on both SQuAD v1 and v2 datasets.

\begin{table*}[htbp]
  \centering
  \caption{Learning rate choices for question-answering dataset SQuAD v1/v2.}
    % \scalebox{0.8}{  no
    \begin{tabular}{l|ccc}
    \Xhline{1.2pt}
    Method      & SQuAD v1 & SQuAD v2  \\
    \hline
    Centralized LoRA & $1.0\times 10^{-3}$ & $5.0\times 10^{-4}$  \\
    HOMLoRA & $1.0\times 10^{-3}$ & $5.0\times 10^{-4}$ \\
    Per-FedAvg-LoRA & $2.0\times 10^{-3}/1.0\times 10^{-3}$  & $1.0\times 10^{-3}/1.0\times 10^{-3}$   \\
    HETLoRA & $5.0\times 10^{-3}$ & $5.0\times 10^{-3}$  \\
    PF2LoRA  & $2.0\times 10^{-3}/1.0\times 10^{-2} $ &$1.0\times 10^{-3}/1.0\times 10^{-2} $    \\
    \Xhline{1.2pt}
    \end{tabular}%
  \label{tab:squadv1_lr}%
\end{table*}%

\begin{table*}[htbp]
  \centering
  \caption{Rank initialization and trainable parameters for DeBERTa v3.}
    % \scalebox{0.8}{  
    \begin{tabular}{l|cc}
    \Xhline{1.2pt}
    Method       & Rank initialization & \# Trainable parameters \\
    \hline
    Centralized LoRA  & $r_k=8$ & 0.30M \\
    HOMLoRA  & $r_k=8$ & 0.30M \\
    Per-FedAvg-LoRA  & $r_k=8$ & 0.30M  \\
    HETLoRA  & $r_{min}=6, r_{max}=14$ & 0.30M \\
    PF2LoRA   &$r_k=8, \tilde{r}_k=4$  & 0.44M  \\
    \Xhline{1.2pt}
    \end{tabular}%
  \label{tab:deberta_hyperparam}%
\end{table*}%

\subsection{GPT-2 on WebNLG and E2E NLG Challenges} \label{sec:hy_nlg} 
The optimal learning rates for each algorithm on WebNLG and E2E NLG Challenges are turned from the range $\{1.0\times10^{-4}, 5.0\times10^{-4}, 1.0\times10^{-3}, 2.0\times10^{-3}, 3.0\times10^{-3}, 4.0\times10^{-3}, 5.0\times10^{-3}\}$, and the learning rate settings are summarized in Table \ref{tab:web_e2e_lr}. 
For the rank initialization, we follow LoRA paper \citep{hu2021lora} and choose a small rank $r_k = 4$ for Centralized LoRA, HOMLoRA, and Per-FedAvg-LoRA. We turn the the best parameters and set $r_{min}=6, r_{max}=12$ for HETLoRA.  PF2LoRA uses the same $r_k=4$ for the common adapter and $\tilde{r_k}=2$ for the client-specific adapter. The detailed rank settings and the number of trainable parameters are shown in Table \ref{tab:gpt2_hyperparam}. HETLoRA sets the sparsity parameter $\gamma=0.99$ and the penalty factor $\lambda=5.0\times 10 ^{-4}$ on both WebNLG and E2E NLG Challenge datasets.

\begin{table*}[htbp]
  \centering
  \caption{Learning rate choices for GPT-2 medium on NLG dataset WebNLG and E2E NLG Challenge.}
    % \scalebox{0.8}{  
    \begin{tabular}{l|ccc}
    \Xhline{1.2pt}
    Method      & WebNLG  & E2E NLG Challenge  \\
    \hline
    Centralized LoRA        & $1.0\times 10^{-3}$                       & $1.0\times 10^{-3}$  \\
    HOMLoRA                 & $1.0\times 10^{-3}$                       & $1.0\times 10^{-3}$ \\
    Per-FedAvg-LoRA         & $2.0\times 10^{-3}/1.0\times 10^{-4}$     & $2.0\times 10^{-3}/2.0\times 10^{-3}$   \\
    HETLoRA                 & $2.0\times 10^{-3}$                       & $2.0\times 10^{-3}$  \\
    PF2LoRA                 & $2.0\times 10^{-3}/1.0\times 10^{-3} $    &$3.0\times 10^{-3}/5.0\times 10^{-4} $    \\
    \Xhline{1.2pt}
    \end{tabular}%
  \label{tab:web_e2e_lr}%
\end{table*}%

\begin{table*}[!h]
  \centering
  \caption{Learning rate choices for GPT2-XL on NLG dataset WebNLG.}
    % \scalebox{0.8}{  
    \begin{tabular}{l|ccc}
    \Xhline{1.2pt}
    Method      & WebNLG   \\
    \hline
    HOMLoRA                 & $1.0\times 10^{-3}$                     \\
    Per-FedAvg-LoRA         & $1.0\times 10^{-3}/1.0\times 10^{-4}$     \\
    HETLoRA                 & $1.0\times 10^{-3}$                       \\
    PF2LoRA                 & $1.0\times 10^{-3}/1.0\times 10^{-4} $      \\
    \Xhline{1.2pt}
    \end{tabular}%
  \label{tab:web_e2e_lr_GPT2_XL}%
\end{table*}%

\begin{table*}[!h]
  \centering
  \caption{Rank initialization and trainable parameters for GPT-2.}
    % \scalebox{0.8}{  
    \begin{tabular}{l|cc}
    \Xhline{1.2pt}
    Method       & Rank initialization & \# Trainable parameters \\
    \hline
    Centralized LoRA  & $r_k=4$ & 0.39M \\
    HOMLoRA  & $r_k=4$ & 0.39M \\
    Per-FedAvg-LoRA  & $r_k=4$ & 0.39M  \\
    HETLoRA  & $r_{min}=6, r_{max}=12$ & 0.81M \\
    PF2LoRA   &$r_k=4, \tilde{r}_k=2$  & 0.59M  \\
    \Xhline{1.2pt}
    \end{tabular}%
  \label{tab:gpt2_hyperparam}%
\end{table*}%

\section{Supplementary Experimental Results for Text Classification} \label{sec:result_roberta_large}
This section provides experimental results about RoBERTa large model on GLUE benchmark. The comparison results with other baselines are shown in Table \ref{tab:result_roberta_large}. We can observe that PF2LoRA achieves higher classification performance. For example, PF2LoRA outperforms HETLoRA by $3.88\%$, $22.24\%$, $2.99\%$, $13.89\%$ and $2.69\%$ on the five datasets, respectively.

\begin{table*}[htbp]
  \centering
  \caption{Roberta-large results on GLUE benchmark. We report "Matthew's correlation" for CoLA and "Accuracy" for MNLI, SST-2, QQP and QNLI. Higher value means "better performance".}
    \setlength{\tabcolsep}{10pt}
    \begin{tabular}{l|ccccc}
    \Xhline{1.2pt}
    Method  & CoLA  & MNLI  & SST-2  & QQP   & QNLI \\
    \hline
    Centralized LoRA & 57.32 & 84.71 & 93.67 & 88.43 & 90.27 \\
    % \hline
    HOMLoRA & 51.71 & 74.51 & 93.33 & 79.76 & 89.63 \\
    % \hline
    Per-FedAvg-LoRA & 51.20  & 75.68 & 92.64 & 81.83 & 79.49 \\
    % \hline
    HETLoRA & 54.15 & 76.38 & 94.53 & 82.55 & 92.31 \\
    % \hline
    PF2LoRA  & \textbf{56.25} & \textbf{93.37} & \textbf{97.36} & \textbf{94.02} & \textbf{94.79} \\
    \Xhline{1.2pt}
    \end{tabular}%
  \label{tab:result_roberta_large}%
\end{table*}%

\section{Supplementary Experimental Results for E2E NLG Challenge} \label{sec:result_e2e}
This section provides experimental results for E2E NLG dataset in Table \ref{tab:result_e2e}. Compared to other federated baselines, our approach demonstrates the best performance on four metrics (BLEU, NIST, ROUGE-L, CIDEr) of five.
\begin{table*}[!h]
  \centering
  \caption{GPT-2 generation results on E2E dataset.}
    \setlength{\tabcolsep}{10pt}
    \begin{tabular}{l|ccccc}
    \Xhline{1.2pt}
    method  & BLEU $\uparrow$  & \multicolumn{1}{l}{NIST $\uparrow$} & \multicolumn{1}{l}{MET $\uparrow$} & \multicolumn{1}{l}{ROUGE-L $\uparrow$} & \multicolumn{1}{l}{CIDEr $\uparrow$} \\
    \hline
    Centralized LoRA & 0.6833 & 8.5321 & 0.4642 & 0.7046 & 2.4023 \\
    % \hline
    HOMLoRA & 0.5585 & 7.0986 & \textbf{0.4349} & 0.6095 & 1.8327 \\
    % \hline
    Per-FedAvg-LoRA & 0.5683 & 7.1190 & 0.4327 & 0.6109 & 1.8984 \\
    % \hline
    HETLoRA &  0.5505     &  7.0088     & 0.4093      & 0.5697      & 1.7167 \\
    % \hline
    PF2LoRA  & \textbf{0.5717} & \textbf{7.1621} & 0.4321 & \textbf{0.6111} & \textbf{1.9088} \\
    \Xhline{1.2pt}
    \end{tabular}%
  \label{tab:result_e2e}%
\end{table*}%

\section{More Ablation Studies.} 
% \subsubsection{Results in Different Heterogeneity Levels} \label{sec:het_cola}
\subsection{The Impact of Heterogeneity Levels}  \label{sec:het_level}
Heterogeneity level is regarded as an important factor in federated learning. In this section, we explore the impact of various heterogeneity levels on the performance of algorithms. We run PF2LoRA and other baselines on text classification datasets SST-2 and QNLI  with three different heterogeneity levels $s=0.6, 0.9, 1.0$. The accuracy results are shown in Table~\ref{tab:het_level}. PF2LoRA performs consistently well on different heterogeneity levels, and HETLoRA follows. The performance of HOMLoRA and Per-FedAvg-LoRA decreases significantly as the heterogeneity level increases. Especially, PF2LoRA outperforms other baselines in a large margin in the case of very high heterogeneity, e.g., $4.35\%$ higher than HETLoRA and $13.87\%$ higher than HOMLoRA on SST-2 dataset. 
% We also show the results in relatively low heterogeneity level ($s=0.2, 0.3, 0.4$) in Table \ref{tab:het_level_cola} in Appendix \ref{sec:het_cola}. Therefore, our algorithm PF2LoRA demonstrates the high robustness to the heterogeneity level. 

\begin{table*}[htbp]
  \centering
  \caption{Results in different heterogeneity levels. We use "Accuracy" to measure the performance here, and higher value means "better performance".}
    \begin{tabular}{l|ccc|ccc}
    \Xhline{1.2pt}
    \multicolumn{1}{l|}{\multirow{2}[1]{*}{Methods}} & \multicolumn{3}{c|}{SST-2} & \multicolumn{3}{c}{QNLI} \\
          & s=0.6 & s=0.9 & s=1.0 & s=0.6 & s=0.9 & s=1.0 \\
    \hline
    HOMLoRA & 92.66 & 92.47 & 83.49 & 86.62 & 85.45 & 67.32 \\
    Per-FedAvg-LoRA & 90.80  & 90.56 & 85.29 & 85.32 & 78.59 & 50.48 \\
    HETLoRA & 93.74 & 93.67 & 91.11 & 89.28 & 91.86 & 89.09 \\
    PF2LoRA  & \textbf{94.12} & \textbf{95.85} & \textbf{95.07} & \textbf{92.87} & \textbf{94.18} & \textbf{93.64} \\
    \Xhline{1.2pt}
    \end{tabular}%
  \label{tab:het_level}%
\end{table*}%

Next, we further study the impact of relatively lower heterogeneity levels on the algorithms.  We run PF2LoRA and other federated baselines on CoLA dataset in the heterogeneity levels of $s=0.2$, $s=0.3$ and $s=0.4$, and the results of  "Matthew's correlation" are summarized in Table \ref{tab:het_level_cola}. PF2LoRA outperforms all the baselines consistently in various heterogeneity levels. For example, PF2LoRA surpasses the best baseline HETLoRA by $4.36\%$, $0.8\%$ and $12.15\%$ in heterogeneity levels of $s=0.2, s=0.3, s=0.4$ respectively. Therefore, our algorithm PF2LoRA demonstrates the high robustness to heterogeneity levels.     

\begin{table}[htbp]
  \centering
  \caption{Matthew's correlation on CoLA in different heterogeneity levels. Higher value means "better performance".}
    \begin{tabular}{l|ccc}
    \Xhline{1.2pt}
    \multirow{2}[2]{*}{Methods} & \multicolumn{3}{c}{CoLA} \\
          & s=0.2 & s=0.3 & s=0.4 \\
    \hline
    HOMLoRA & 52.91 & 50.75 & 43.17 \\
    Per-FedAvg-LoRA & 53.48 & 51.11 & 44.44 \\
    HETLoRA & 53.86 & 53.76 & 45.03 \\
    PF2LoRA  & \textbf{56.20}  & \textbf{54.19} & \textbf{50.50} \\
    \Xhline{1.2pt}
    \end{tabular}%
  \label{tab:het_level_cola}%
\end{table}%

\subsection{Performance with/without Bilevel Optimization} \label{sec:ablation_BO}
We conduct an ablation study to verify the effect of bilevel optimization. Instead of applying bilevel optimization in \eqref{eq:bilevel_obj}, we update parameters in the common and client-specific adapters simultaneously. 
\begin{equation}\label{eq:update_sim}
    \begin{aligned}
        &\min_{x, y_k} \frac{1}{M}\sum_{k=1}^{M}f_k(x, y_{k}), \\
        &f_k(x, y_{k}) \coloneq \mathbb{E}_{\xi\sim \mathcal{D}_k} F_k(x, y_k; \xi),
    \end{aligned}
\end{equation}
where $\mathcal{D}_k$ is the data on client $k$.
Specifically, we keep the optimizer settings mentioned in Section~\ref{sec:robert_exp_detail}, where a SGD optimizer is applied to updating the client-specific adapter and an AdamW optimizer to the common adapter. The difference is that we do not use the hypergradient \eqref{eq:hypergrad} to update the common adapter, instead update it by $x_{k}^{t+1} = x_k^t - \eta\nabla_xF_k(x_k^t, y_k^{t};\xi_k^t)$. We execute our ``two-level low rank adaptation" framework without bilevel optimization on text classification of GLUE benchmark. For fair comparison, we keep the same hyperparameter settings as that in Section \ref{sec:robert_exp_detail}, including heterogeneity level, learning rates, communication rounds, communication interval and initial rank dimension on the same dataset. The comparison results are shown in Figure \ref{fig:ablation_bo}, where we can see that the framework with bilevel optimization (BO) always performs better than that without BO, especially on harder classification task, such as CoLA dataset.

\begin{figure}[!t]
    \centering
    \includegraphics[scale=0.28]{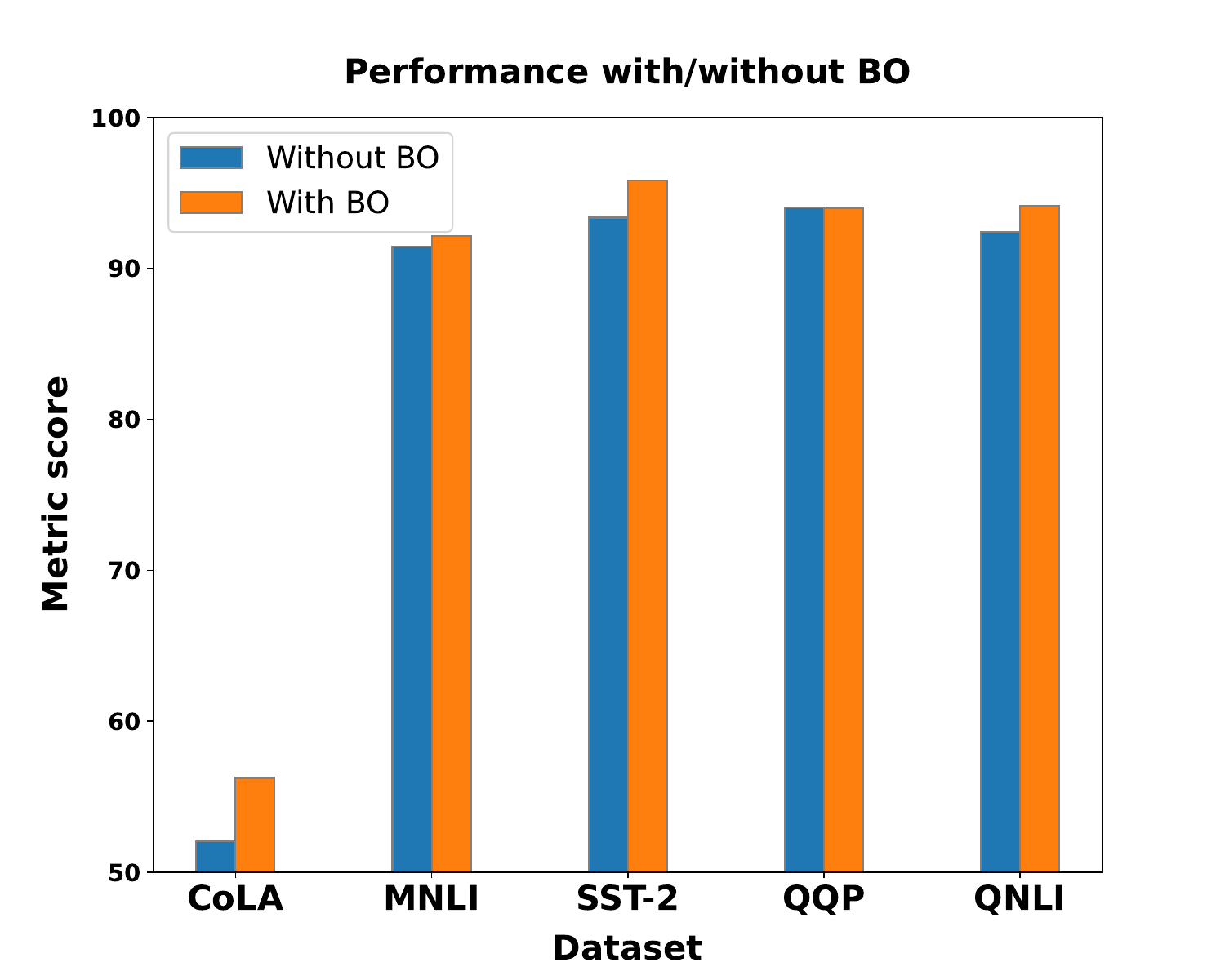}
     \caption{Performance comparison with/without bilevel optimization (BO).  We report "Matthew's correlation" for CoLA and "Accuracy" for MNLI, SST-2, QQP and QNLI. Higher score means "better performance"}    
    \label{fig:ablation_bo}
\end{figure}

\section{Stability Analysis} \label{sec:stability}
\begin{figure}[!t]
    \centering
    \subfigure[The averaged loss.]{\includegraphics[scale=0.4]{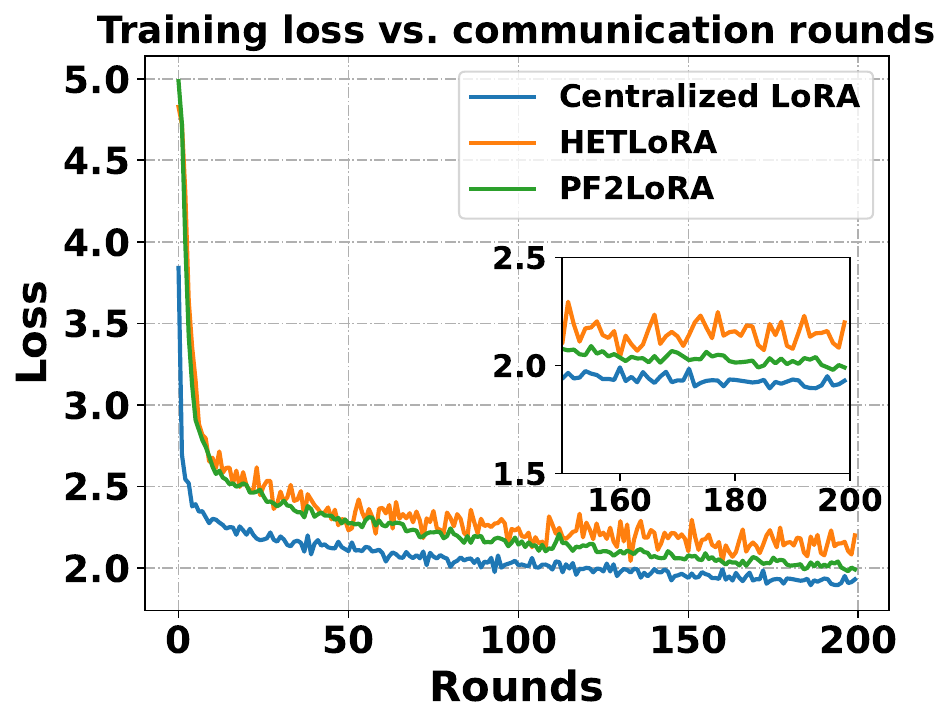}\label{fig:ppl_webnlg}}
    \subfigure[The averaged perplexity.]{\includegraphics[scale=0.4]{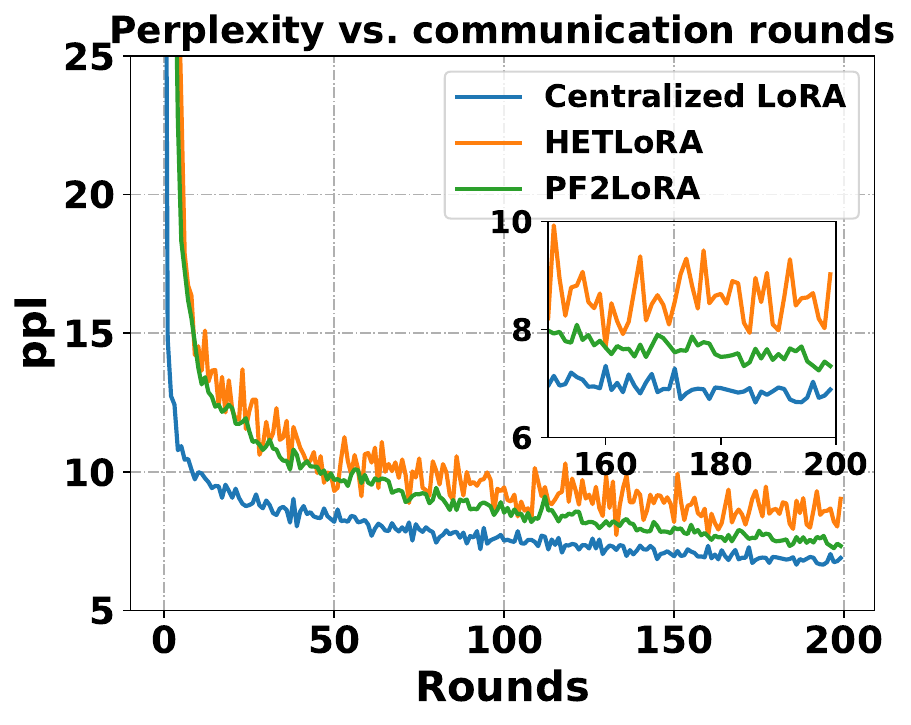}    \label{fig:ppl_webnlg}}
     \caption{The averaged training loss and perplexity on natural language generation task of WebNLG.}    
    \label{fig:loss_ppl_webnlg}
\end{figure}
 % Hetlora is unstable (front)
% Our two-level low rank adaption framework shows the strong performance on multiple NLG tasks, and we also observe that our method exhibit better training stability.
Despite that HETLoRA is a strong baseline which performs usually well on heterogeneous data. However, we empirically observe that the training process of HETLoRA is not as stable as ours and Centralized LoRA in Figure \ref{fig:loss_ppl_webnlg}, where the training loss and perplexity (ppl) are averaged across all the clients. A possible and reasonable explanation is that HETLoRA adopts dynamical rank pruning and matrices truncation which directly change the intrinsic structure of local adapters, leading to unstable training. On the one hand, pruning removes some columns or rows from the original weights, which can degrade the model performance and require some steps of fine-tuning to recover the performance \citep{han2015deep}. On the other hand, each client is required to truncate the common adapter matrices to align the matrices' dimensions at each communication round, which inevitably loses some potentially important information. In contrast, our method circumvents the alignment issue of adapter matrices by assigning a uniform rank $r_k$ to the common adapter and uniform $\tilde{r}_k$ to all the client-specific adapters.

% \subsection{The Analysis of Rank and Heterogeneity}
% Our adapter $x_i$, $y_i$ is important. Compare ours with 
%     Heterloram, for lower-performance clients, HETLoRA has small rank; we have small gap of performance over clients (plot distribution of rank)
%     heterlora will keep the order of rank as in initialization, our method can automatically adjust the rank based on the data.

\section{Generated Result of NLU}
\subsection{Generated Examples for E2E NLG Challenge}
Table \ref{tab:e2e_generated_example1} and \ref{tab:e2e_generated_example2} show the generated examples of algorithm HETLoRA and PF2LoRA. The federated fine-tuning experiments are run across 8 clients on E2E NLG Challenges, where we construct the heterogeneous data by the "name" of restaurants, thus each client has different meta-information from different restaurants. There are 18 restaurants in the test set distributed in 8 clients. We show the generated examples based given context information on each client, while multiple references are provided to evaluate the quality of generated contents. We compare the generated contents from HETLoRA and PF2LoRA. In most cases, PF2LoRA can generate more complete and logically coherent sentences.
For example, the generated contents on client 4 and client 7, HETLoRA misses some important information (highlighted in green). The examples on  client 1, 2, 3 and 4, PF2LoRA produces more grammatically coherent sentences than HELoRA.  
% For example, the given context in client 7 is some meta-information about a pub, including the name, the type, the food, the location, etc. The content generated by HETLoRA losses some important information, such as the restaurant name. In contrast, PF2LoRA correctly summarizes the provided meta-information.   

\begin{table*}[htbp]
  \centering
  \caption{The generated examples for E2E NLG Challenges}
  \scalebox{0.8}{
    \begin{tabular}{ll}
    \toprule
    \toprule
    \multicolumn{2}{c}{Client 0} \\
    \midrule
    Context  & name : blue spice | type : pub | food : english | area : riverside | family friendly : yes | near : rainbow vegetarian café \\
    \midrule
    References  & in riverside , near the rainbow vegetarian café , you can find a family friendly pub called blue spice . \\
          & if you like english food there is a family - friendly pub called blue spice near the rainbow vegetarian café in riverside . \\
          & the blue spice is a child - friendly , english pub located in riverside area , near rainbow vegetarian café . \\
          & blue spice is located near rainbow vegetarian café in the riverside area and is a kid friendly pub that serves \\& english food . \\
          & there is a pub called blue spice which serves english food , is kid friendly , and is in riverside near rainbow \\&vegetarian café . \\
          & blue spice is a child - friendly pub near rainbow vegetarian café in the riverside area . \\
          & blue spice near rainbow vegetarian café in riverside is a pub serving english meals and child friendly \\
          & the blue spice is a pub . it is located near rainbow vegetarian café in the area of riverside . this is a family \\&friendly pub \\
          & serving english food . \\
          & an english serving child friendly pub in riverside is blue spice near rainbow vegetarian café \\
          & there is a pub that provides food and is children friendly , near rainbow vegetarian café and the riverside and is \\&called blue spice . \\
          & situated near the rainbow vegetarian café in the riverside area of the city , the blue spice pub , is ideal if you fancy\\
          &traditional english food whilst out with the kids . \\
    \midrule
    HETLoRA & blue spice is a pub near rainbow vegetarian café in the riverside area . \textcolor{red}{it is family friendly} \\&\textcolor{red}{and serves english food }. \\
    PF2LoRA & \textcolor{green}{blue spice is a family friendly pub that serves english food }. it is located in the riverside area \\&near the rainbow vegetarian café . \\
    \midrule
    \midrule
    \multicolumn{2}{c}{Client 1} \\
    \midrule
    Context  & name : the cricketers | type : coffee shop | customer rating : low | family friendly : no | near : ranch \\
    \midrule
    References  & the cricketers is a coffee shop with a low customer rating , located near ranch . it is not family - friendly . \\
    \midrule
    HETLoRA & \textcolor{red}{city centre coffee shop , the cricketers} , is not family - friendly and has a low customer rating . it is located near ranch . \\
    PF2LoRA & north of ranch , \textcolor{green}{there is a coffee shop called the cricketers} . it is not family - friendly and has a low customer rating . \\
    \midrule
    \midrule
    \multicolumn{2}{c}{Client 2} \\
    \midrule
    Context  & name : the mill | type : restaurant | food : english | price : moderate | customer rating : 3 out of 5 | \\ 
    &area : riverside | family friendly : yes | near : café rouge \\
    \midrule
    References  & the riverside area has restaurant near the café rouge that is both in the moderate price range and kid friendly \\
    &called the mill . it has a 3 out of 5 customer rating and serves english food . \\
          & the riverside area near café rouge has a restaurant that is kids - friendly . it has a price range in the mill . i give\\ &the food a 3 out of 5 . \\
          & the mill is a kids friendly restaurant that has moderate prices and serves english food . it has a 3 out of 5 customer\\ &rating and is located in the riverside area near the café rouge . \\
    \midrule
    HETLoRA & the mill is a moderately priced english restaurant \textcolor{red}{near café rouge in the riverside area} . it is kid friendly and has\\& a customer rating of 3 out of 5 . \\
    PF2LoRA & the mill is a moderately priced restaurant \textcolor{green}{in the riverside area near café rouge} . it serves english food and is kid \\&friendly . it has a customer rating of 3 out of 5 . \\
    \midrule
    \midrule
    \multicolumn{2}{c}{Client 3} \\
    \midrule
    Context  & name : the phoenix | type : pub | food : french | price : £ 20 - 25 | customer rating : high | area : riverside \\&| family friendly : no | near : crowne plaza hotel \\
    \midrule
    References  & a pub that is not kid friendly is located in the riverside area near crowne plaza hotel . it is named the phoenix\\& , has french food and price range of £ 20 - £ 30 and a high customer rating . \\
          & the phoenix , which is a pub that is not kid friendly , is near crowne plaza hotel and serves french food in the price \\&range of £ 20 - 25 in the riverside area . it has a high customer rating . \\
    \midrule
    HETLoRA & the phoenix is a pub \textcolor{red}{near the crowne plaza hotel in the riverside area} . it has a high customer rating and a price \\&range of £ 20 - 25 . it is not kid friendly . \\
    PF2LoRA & the phoenix is a pub \textcolor{green}{in the riverside area near the crowne plaza hotel} . it serves french food with a price range of \\&£ 20 - 25 and has a high customer rating . it is not kid friendly . \\
    
    \bottomrule
    \bottomrule
    \end{tabular}}%
  \label{tab:e2e_generated_example1}%
\end{table*}%

\begin{table*}[htbp]
  \centering
  \caption{The generated examples for E2E NLG Challenges (continued).}
  \scalebox{0.8}{
    \begin{tabular}{ll}
    \toprule
    \toprule
    \multicolumn{2}{c}{Client 4} \\
    \midrule
    Context  & name : the punter | type : restaurant | food : italian | price : cheap | customer rating : average | area : riverside\\& | family friendly : no | near : rainbow vegetarian café \\
    \midrule
    References  & hello and welcome to the punter , we serve the finest italian food around and have an average customer rating this\\& is very good for a restaurant we are near rainbow vegetarian café and our area is the riverside our price range is very\\& cheap for such good food at the moment we are not family - friendly . \\
          & a restaurant serving italian food for adults can be found on the riverside near rainbow vegetarian café . the punter has \\&average ratings , and cheap prices \\
    \midrule
    HETLoRA & the italian restaurant the punter is located in the riverside area near rainbow vegetarian café . it is not \\&family - friendly and has an average customer rating . \\
    PF2LoRA & the punter is a \textcolor{green}{cheap} italian restaurant near the rainbow vegetarian café in the riverside area . it has an average\\& customer rating and is not family - friendly . \\
    \midrule
    \midrule
    \multicolumn{2}{c}{Client 5} \\
    \midrule
    Context  & name : the vaults | type : pub | food : japanese | price : less than £ 20 | customer rating : average \\&| area : city centre | family friendly : no | near : raja indian cuisine \\
    \midrule
    References  & no , the vaults is fast food with pub on side raja indian cuisine in the city centre with price less than \\&£ 20 average \\
          & the vaults pub that is less than 20 pounds and has an average customer rating and is family - friendly . they serve\\& japanese food near the center of the city and also near the raja indian cuisine . \\
    \midrule
          & the vaults is japanese with pub on side raja indian cuisine in the city centre with average no less than £ 20 \\
    \midrule
    HETLoRA & \textcolor{red}{the vaults is a pub that serves japanese food . it is located in the city centre near raja indian cuisine }. it is not\\& family - friendly and has a price range of less than £ 20 . it has an average customer rating . \\
    PF2LoRA & \textcolor{green}{the vaults is a japanese pub located in the city centre near raja indian cuisine }. it has an average customer rating \\&and a price range of less than £ 20 . it is not family - friendly . \\
    \midrule
    \midrule
    \multicolumn{2}{c}{Client 6} \\
    \midrule
    Context  & name : the waterman | type : pub | food : italian | price : high | area : riverside | family friendly : yes \\&| near : raja indian cuisine \\
    \midrule
    References  & the waterman is a high price range italian pub in the riverside area near raja indian cuisine . it is children\\& friendly .\\
    & near raja indian cuisine , the waterman pub has high prices and facilities for the family .\\
    & the pub the waterman is an expensive family friendly establishment located near raja indian cuisine . \\
    & located near raja indian cuisine in the riverside area , the waterman is a high priced children friendly pub serving italian \\&food . \\
    & conveniently located on the riverside , close to raja indian cuisine , the waterman is an elite , but family friendly \\&established pub . \\
    \midrule
    HETLoRA & the waterman is a high priced italian pub near raja indian cuisine in the riverside area . \textcolor{red}{it is children friendly }. \\
    PF2LoRA & the waterman is an italian pub located in the riverside area near raja indian cuisine . \textcolor{green}{it is children friendly and} \\& \textcolor{green}{has a high price range} . \\
    \midrule
    \midrule
    \multicolumn{2}{c}{Client 7} \\
    \midrule
    Context  & name : name : wildwood | type : pub | food : indian | area : city centre | family friendly : yes \\&| near : raja indian cuisine \\
    \midrule
    References  &located near the city centre and raja indian cuisine the wildwood is a family friendly indian pub .\\
    & wildwood is in the city centre area near raja indian cuisine . it is a pub that serves indian food and is family friendly .   \\
    & wildwood also offers indian food to go along with the family friendly pub located near raja indian cuisine\\
    \midrule
    HETLoRA & aji indian cuisine pub in the city centre near raja indian cuisine is kid friendly and serves indian food . \\
    PF2LoRA & a pub near raja indian cuisine in the city centre called \textcolor{green}{ wildwood} serves indian food and is kid friendly . \\
    \bottomrule
    \bottomrule
    \end{tabular}}%
  \label{tab:e2e_generated_example2}%
\end{table*}%

\subsection{Generated Examples for WebNLG}
For WebNLG dataset, we construct the heterogeneity data by the topics [`Airport', `Astronaut', `Building', `City', `ComicsCharacter', `Food', `Monument', `SportsTeam', `University', `WrittenWork']. These topics are distributed across 8 clients. Thus, the language style varies with the text topics. We run the personalized federated fine-tuning across 8 clients and report the generated examples for given test context. The comparison results show that PF2LoRA can generate more complete and high quality sentences than HETLoRA. For example on client 0 and 1,  HETLoRA misses key words ``runwayname",  ``test pilot", which actually are important information. On client 2 and 5, HETLoRA generates incorrect information, while PF2LoRA produces accurate sentences.

\begin{table*}[htbp]
  \centering
  \caption{The generated examples for WebNLG.}
  \scalebox{0.8}{
    \begin{tabular}{ll}
    \toprule
    \toprule
    \multicolumn{2}{c}{Client 0 (Airport)} \\
    \midrule
    Context  & al\_asad\_airbase : operatingorganisation : united\_states\_air\_force | al\_asad\_airbase : runwaylength : 3992 . 88 | \\&al\_asad\_airbase : location : " al anbar province , iraq " | al\_asad\_airbase : icao\_location\_identifier : " oraa " | \\&al\_asad\_airbase : runwayname : " 08 / 26 " \\
    \midrule
    References  & al asad air base has a runway name of 08 / 26 which is 3992 . 8 in length . it is situated in the al anbar \\& province of iraq , is operated by the united states air force and has the icao location identifier oraa . \\
    & the united states airport operates the al asad airbase which is located in the al anbar province , iraq . the icao location\\& identifer of al asad airbase is oraa and the length is 3992 . 88m and the runway is known as 08 / 26 .\\
    \midrule
    HETLoRA & ! the united states air force is the operating organisation for al asad airbase which is located in al anbar province \\&, iraq . the airbase has a runway length of 3992 . 88 and the icao location identifier is oraa . \\
    PF2LoRA & the united states air force is the operating organisation for al asad airbase in al anbar province , iraq . the icao \\& location identifier of al asad airbase is oraa and it has a runway length of 3992 . 88 . \textcolor{green}{the runway name} \\&\textcolor{green}{of the airbase is 08 / 26 }. \\
    \midrule
    \midrule
    \multicolumn{2}{c}{Client 1 (Astronaut)} \\
    \midrule
    Context  & alan\_shepard : status : " deceased " | alan\_shepard : almamater : " nwc , m . a . 1957 " | alan\_shepard : deathplace \\&: california | alan\_shepard : occupation : test\_pilot | alan\_shepard : birthplace : new\_hampshire | alan\_shepard : was \\&selected by nasa : 1959 | alan\_shepard : birthdate : " 1923 - 11 - 18 "\\
    \midrule
    References  & alan shepard has died in california . he was born on 18 november 1923 in new hampshire and attended school at nwc\\& , graduating in 1957 with an ma . he became a test pilot and was eventually selected by nasa in 1959 . \\
    & alan shepard was born in new hampshire on november 18th , 1923 . he graduated from nwc in 1957 with an m . a . \\&he was selected by nasa in 1959 and he was a test pilot . he died in california .\\
    & alan shepard , born on november 18 , 1923 , graduated from nwc in 1957 with an m . a . alan shepard served as a \\&test pilot , and was selected by nasa in 1959 . alan shepard , born in new hampshire , died in california , .\\
    \midrule
    HETLoRA & alan shepard was born on november 18th , 1923 in new hampshire . he graduated from nwc in 1957 with an m . a . \\&and was selected by nasa in 1959 . he died in california .\\
    PF2LoRA & alan shepard was born in new hampshire on november 18th , 1923 . he graduated from nwc with a m . a . in 1957 . \\&he was  selected by nasa in 1959 and served as \textcolor{green}{a test pilot} . alan shepard died in california . \\
    \midrule
    \midrule
    \multicolumn{2}{c}{Client 2 (Building)} \\
    \midrule
    Context  & adisham\_hall : country : sri\_lanka | sri\_lanka : capital : sri\_jayawardenepura\_kotte | sri\_lanka : currency : \\&sri\_lankan\_rupee \\
    \midrule
    References  & sri jayawardenepura kotte is the capital of sri lanka , which uses the sri lankan rupee as its currency\\& and is the location of adisham hall .\\
    & sri jayawardenepura kotte is the capital of sri lanka , whose currency is the rupee . adisham hall is located \\&in sri lanka .\\
    \midrule
    HETLoRA & \textcolor{red}{college adisham hall} is located in the country of sri lanka , where the capital is sri jayawardenepura kotte \\&and the currency is the sri lankan rupee . \\
    PF2LoRA & \textcolor{green}{alan adisham hall} is located in sri lanka , the capital of which is sri jayawardenepura kotte . the currency of sri lanka \\& is the sri lankan rupee . \\
    \midrule
    \midrule
    \multicolumn{2}{c}{Client 3 (File)} \\
    \midrule
    Context  & big\_hero\_6\_ ( film ) : starring : ryan\_potter | big\_hero\_6\_ ( film ) : distributor : \\&walt\_disney\_studios\_motion\_pictures | baymax : series : big\_hero\_6\_ ( film ) \\
    \midrule
    References  & the movie big hero 6 stars ryan potter which has baymax as one of its characters , was distributed by walt disney \\&studios motion pictures .\\
    & baymax is a character in the big hero 6 film starring ryan potter and distributed by walt disney studios motion pictures .   \\
    & walt disney studio motion picture distributed the film big hero 6 , in which ryan potter starred and baymax is a character .\\
    \midrule
    HETLoRA & \textcolor{red}{!} baymax is a character in the film big hero 6 which stars ryan potter . the film was distributed by walt disney \\&studios motion pictures . \\
    PF2LoRA & walt disney studios motion pictures is the distributor of big hero 6 , a film in which baymax is a character . \\&the film stars ryan potter . \\
    \bottomrule
    \bottomrule
    \end{tabular}}%
  \label{tab:webnlg_generated_example1}%
\end{table*}%

\begin{table*}[htbp]
  \centering
  \caption{The generated examples for WebNLG (continued).}
  \scalebox{0.8}{
    \begin{tabular}{ll}
    \toprule
    \toprule
    \multicolumn{2}{c}{Client 4 (Food)} \\
    \midrule
    Context  & bacon\_sandwich : dishvariation : blt | bacon\_sandwich : mainingredients : " bread and bacon , with a condiment , \\&often ketchup or brown sauce " | bacon\_sandwich : country : united\_kingdom | bacon\_sandwich : ingredient : ketchup | \\&bacon\_sandwich : alternativename : " bacon butty , bacon sarnie , rasher sandwich , bacon sanger , piece ' n bacon , bacon \\&cob , bacon barm , bacon muffin " \\
    \midrule
    References  & the bacon sandwich , also known as : bacon butty , bacon sarnie , rasher sandwich , bacon sanger , piece n ' \\&bacon , bacon cob , bacon barm , blt and bacon muffin , is from the united kingdom . it includes bread , bacon \\&and ketchup or brown sauce . \\
    & the variant blt bacon sandwich includes : bread , bacon , and ketchup or brown sauce . it originates from the united kingdom\\& and goes by the various names : bacon butty , bacon sarnie , rasher sandwich , bacon sanger , piece n ' bacon , \\&bacon barm , bacon cob and bacon muffin .\\
    \midrule
    HETLoRA & the bacon sandwich is a variation of the blt . the main ingredients are bread and bacon , with a condiment , often\\& ketchup or brown sauce . it is also known as a bacon butty , bacon sarnie , rasher sandwich , bacon sanger , piece ' n bacon , \\&bacon cob , bacon barm , bacon muffin .\\
    PF2LoRA & the bacon sandwich is a variation of the blt \textcolor{green}{and is a dish from the united kingdom }. the main ingredients of a bacon \\&sandwich are bread and bacon with ketchup or brown sauce as a condiment . the bacon sandwich can also be known \\&as a bacon butty , bacon sarnie , rasher sandwich , bacon sanger , piece ' n bacon , bacon cob , bacon barm , or bacon muffin .\\
    \midrule
    \midrule
    \multicolumn{2}{c}{Client 5 (SportsTeam)} \\
    \midrule
    Context  & a . f . c . \_blackpool : manager : stuart\_parker\_ ( footballer ) | stuart\_parker\_ ( footballer ) : club : \\&chesterfield\_f . c . | stuart\_parker\_ ( footballer ) : club : stockport\_county\_f . c . | a . f . c . \_blackpool : ground :\\& blackpool | blackpool : leader : labour\_party\_ ( uk )\\
    \midrule
    References  & a . f . c . blackpool is in blackpool , which council is labour , it has had stuart parker as their manager ,\\& whose football club was stockport county f . c and is attached to chesterfield football club .\\
    & alan shepard , born on november 18 , 1923 , graduated from nwc in 1957 with an m . a . alan shepard served as a \\&test pilot , and was selected by nasa in 1959 . alan shepard , born in new hampshire , died in california , .\\
    \midrule
    HETLoRA & ! stuart parker ( footballer ) is the manager of a . f . c . blackpool \textcolor{red}{who play in blackpool} , where the leader is the\\& labour party ( uk ) and the ground is called blackpool .\\
    PF2LoRA & ! a . f . c . blackpool is in blackpool , where the leader is the labour party ( uk ) . the club is managed by \\&stuart parker ( footballer ) \textcolor{green}{who played for chesterfield fc and stockport county f . c .} \\
    \midrule
    \midrule
    \multicolumn{2}{c}{Client 6 (University)} \\
    \midrule
    Context  & romania : ethnicgroup : germans\_of\_romania | romania : leadertitle : prime\_minister\_of\_romania | alba\_iulia : \\&country : romania | romania : leadername : klaus\_iohannis | romania : capital : bucharest | 1\_decembrie\_1918\_university :\\& city : alba\_iulia | romania : anthem : deșteaptă - te , \_române ! \\
    \midrule
    References  & the 1 decembrie 1918 university is in the city alba iulia in romania . klaus iohannis the leader of romania and \\&they also have a prime minister . the germans of romania are the main ethnic group in romania and the capital is bucharest . \\&the romania anthem is deșteaptă - te , române !\\
    \midrule
    HETLoRA &  \textcolor{red}{!} the 1 decembrie 1918 university is located in alba iulia , romania . the country ' s leader is prime minister klaus\\& iohannis and its capital is bucharest . the anthem of the country is deșteaptă - te , române ! \\
    PF2LoRA & the 1 decembrie 1918 university is located in alba iulia , romania . romania ' s capital is bucharest and its leader \\&is prime minister klaus iohannis . the national anthem of romania is deșteaptă - te , române ! \textcolor{green}{and its ethnic group is the}\\& \textcolor{green}{germans of romania }. \\
    \midrule
    \midrule
    \multicolumn{2}{c}{Client 7 (WrittenWork)} \\
    \midrule
    Context  & administrative\_science\_quarterly : publisher : cornell\_university | cornell\_university : affiliation : \\&association\_of\_public\_and\_land - grant\_universities | cornell\_university : affiliation : \\&association\_of\_american\_universities | cornell\_university : president : elizabeth\_garrett | cornell\_university : city :\\& ithaca , \_new\_york \\
    \midrule
    References  & administrative science quarterly was published by cornell university , located in ithaca , new york , and \\&affiliated with the association of public and land grant universities , as well as with the association of american \\&universities . president of cornell university is elizabeth garrett .\\
    \midrule
    HETLoRA &\ \textcolor{red}{!} the administrative science quarterly is published by cornell university , which is affiliated with the association of\\& public and land grant universities and the association of american universities . it is located in ithaca , new york . the \\& \textcolor{red}{president of} cornell university is elizabeth garrett . \\
    PF2LoRA & the administrative science quarterly is published by cornell university , ithaca , new york . the university is\\& affiliated with the association of public and land grant universities and the association of american universities . \textcolor{green}{the} \\& \textcolor{green}{president of} the university is elizabeth garrett .\\
    \bottomrule
    \bottomrule
    \end{tabular}}%
  \label{tab:webnlg_generated_example2}%
\end{table*}%

\section{Proof of Theorem~\ref{mainthm1}}
\subsection{Basic Lemmas}
The hypergradient estimation  is defined as $\nabla \widehat{\Phi} (x;y^{t+1})=\nabla_x f(x,y^{t+1})-\alpha\nabla_{xy}f(x,y^t)\nabla_y f(x,y^{t+1})$.
\begin{lemma}[gradient descent for strongly convex and smooth functions]\label{scdecent}
when $\alpha \leq \frac{1}{L_{f,1}}$, for lower level each step we have
\begin{equation}
    \|y^{t+1}-y^*(x^t)\| \leq (1-\alpha\mu)^{\frac{1}{2}}\|y^t-y^*(x^t) \|.
    %\leq (1-\alpha\mu)^{\frac{t+1}{2}}\|y_k^0-y^*(x_k) \|
\end{equation}
\end{lemma}

\begin{proof}
Note that
\begin{align}
    &\|y^{t+1}-y^*(x^t)\|^2 
    =
    \|y^t-\alpha\nabla_y f(x^t,y^t)-y^*(x^t)\|^2 
    \\\nonumber
    &= \|y^t-y^*(x^t) \|^2-2\alpha\langle \nabla_y f(x^t,y^t),y^t-y^*(x^t) \rangle
    +\alpha^2\|\nabla_y f(x^t,y^t)\|^2 \nonumber\\
    &\overset{(i)}{\leq} (1-\alpha\mu)\|y^t-y^*(x^t) \|^2-2\alpha(f(x,y^t)-\inf_y f(x^t,y))+\alpha^2\|\nabla f_y(x^t,y^t)\|^2 \nonumber\\
    &\overset{(ii)}{\leq} (1-\alpha\mu)\|y^t-y^*(x^t) \|^2-2\alpha(f(x^t,y^t)-\inf_y f(x^t,y))+2\alpha^2L_{f,1}(f(x^t,y^t))-\inf_y f(x^t,y)) \nonumber\\
    &=(1-\alpha\mu)\|y^t-y^*(x^t) \|^2-2\alpha(1-\alpha L_{f,1})(f(x^t,y^t)-\inf_y f(x^t,y))\nonumber\\
    &\overset{(iii)}{\leq} (1-\alpha\mu)\|y^t-y^*(x^t) \|^2
\end{align}
where $(i)$ is because of the $\mu$-strongly convexity, $(ii)$ is because of $L_{g,1}$-smooth of the function, $(iii)$ is because of $2\alpha(1-\alpha L_{f,1})(f(x^t,y^t)-\inf_y f(x^t,y)) \geq 0$.
\end{proof}
%\begin{lemma}
%    $\Phi(x)$ is $4L+\alpha\rhoB$-smooth
%\end{lemma}

\begin{lemma}[true hypergradient] \label{hyper-formulation}
The hypergradient $\nabla \Phi(x)$ equals to $\nabla_x f(x,y^*(x))$.
\end{lemma}

\begin{proof} By the implicit function theorem~\citep{ghadimi2018approximation}, we have
\begin{equation*}
    \begin{aligned}
        &\nabla\Phi(x) = \nabla_x f(x,y^*(x)) - \nabla_{xy} f(x,y^*(x))[\nabla_{yy} f(x,y^*(x))]^{-1}\nabla_y f(x, y^*(x))
        \overset{(i)}{=}\nabla_x f(x,y^*(x))
%        &\nabla y^*(x_k)= \nabla_{xy} f(x_k,y^*(x_k))[\nabla_{yy} f(x_k,y^*(x_k))]^{-1}\\
    \end{aligned}
\end{equation*}
where $(i)$ holds due to $\nabla_y f(x,y^*(x))=0$.
\end{proof}

%where $z_k^*(x_k)$ is the solution to the following linear system:
%begin{equation*}
%    z^*(x) = \argmin_z \frac{1}{2}\langle \gdyy g(x,y^*(x))z, z \rangle - \langle \gdy f(x,y^*(x)), z \rangle.
%\end{equation*}

\begin{lemma}[Lipschitz property~\citep{ghadimi2018approximation}]\label{lipschitz}
 $y^*(x)$ is $\frac{L_{f,1}}{\mu}$-Lipschitz continuous.
 %and $\nabla y^*(x)$ is $\frac{L_{g,1}L_{g,2}^2 \Big\rho}{\mu^2}$-Lipschitz continuous. 
\end{lemma}
 \begin{lemma}[Lipschitz hypergradient]
    $\Phi(x)$ is $L_\Phi$-smooth and $L_\Phi= L_{f,1}+\frac{L_{f,1}^2}{\mu}$.
 \end{lemma} 

\begin{proof}
By definition of hypergradient in Lemma~\ref{hyper-formulation} and Assumption~\ref{ass:bilevel}, we have
\begin{align}
    &\|\nabla\Phi(x_1)-\nabla\Phi(x_2)\| 
   = \|\nabla_x f(x_1,y^*(x_1)-\nabla_x f(x_2,y^*(x_2)\| \nonumber\\
    &\leq L_{f,1}\|x_1-x_2\|+L_{f,1}\|y^*(x_1)-y^*(x_2)\|\nonumber \\
    &\overset{(i)}{\leq} L_{f,1}\|x_1-x_2\|+\frac{L_{f,1}^2}{\mu}\|x_1-x_2\|=L_\Phi\|x_1-x_2\|,
\end{align}
where $(i)$ comes from Lemma~\ref{lipschitz}.
\end{proof}

 \subsection{Proof}
%  \begin{comment}
% \begin{lemma}[Hypervariance]
% For each step variance we have
% \begin{align}
%     &\E_t[\|[\hat{\Phi}(x_k;y_k^{t+1};\xi_t',\zeta_t')] - \E_t\hat{\Phi}(x_k;y_k^{t+1};\xi_t',\zeta_t')\mid F_k] \|^2\mid F_k] \nonumber\\
%     &\leq \sigma_{L,1}^2+ \sigma_{L,2}^2\|\nabla_{xy}f_{k}(x^t_k, y^{t+1}_{k})\|^2 \leq \sigma_{L,1}^2+ \sigma_{L,2}^2L_{g,1}^2
% \end{align}
% \begin{proof}
% \begin{align}
% &\E_t \hat{\Phi}(x_k;y_k^{t+1};\xi_t',\zeta_t')
%     =\nabla \widehat{\Phi}(x_k;y_k^{t+1})
%     = \nabla_x f(x_k, y^{t+1}_{k})  
%     -\alpha \nabla_{xy}f(x_k, y^{t}_k)\nabla_y f(x_k, y^{t+1}_{k}) \nonumber \\
%    & \hat{\Phi}(x_k;y_k^{t+1};\xi_t',\zeta_t')
%     =\nabla_x F_{k}(x^{t}_{k}, y^{t+1}_{k}; \xi_k^t) - \alpha \nabla_{xy}F_{k}(x^{t}_k, y^{t}_{k}; \zeta_k^t)\nabla_yF_{k}(x^t_k, y^{t+1}_{k};\tilde{\xi}_k^t) \nonumber
% \end{align}
% \end{proof}
% (do not need the lemma under large batch setting)
% %which is because $\nabla_yf_{k}(x^t_k, y^{t+1}_{k})$ not in sigma algerba.
% \end{lemma}
% \end{comment} 

\begin{lemma}[Hypergradient bias] \label{lm:hypergradient-bias}
% \begin{small}
%$\Phi(x_k)$ is $L_{f,1}+L_{f,2}L_{g,1}$ smooth
 Hypergradient estimation $\nabla \widehat{\Phi} (x;y^{t+1})$ satisfy:
\begin{equation}
    \begin{aligned}
        & \| \nabla \widehat{\Phi}(x^t;y^{t+1}) - \nabla\Phi(x^t)\| \leq L_{f,1}(\alpha L_{f,1}+1)(1-\alpha\mu)^{\frac{1}{2}}\|y^{t}-y^*(x^t)\| \nonumber\\
 %   &= A\|y_{k}^{t}-y^*(x_k)\|
    \end{aligned}
\end{equation}
\end{lemma}

\begin{proof}
Note that
\begin{align}
     &\nabla \widehat{\Phi}(x^t;y^{t+1}) - \nabla\Phi(x^t) \nonumber \\
    &= \nabla_x f(x^t, y^{t+1})  
    - \alpha\nabla_{xy}f(x^t, y^{t})\nabla_y f(x^t, y^{t+1})-\nabla_x f(x^t,y^*(x^t))\nonumber \\
    &\overset{(i)}{=} \nabla_x f(x^t, y^{t+1})-\nabla_x f(x^t, y^*(x^t))- \alpha\nabla_{xy}f(x^t, y^{t})(\nabla_y f(x^t, y^{t+1})-\nabla_y f(x^t,y^*(x^t)))
 %      &- \nabla_{xy} f_k(x_k,y_k^*(x_k))[\nabla_{yy} f_k(x_k,y_k^*(x_k))]^{-1}\nabla_y f_k(x_k, y_k^*(x_k))]\\
 %       &\quad= [\nabla_x f_k(x_t,y_t) - \nabla_x f_k(x_t,y_t^*)] - \{(\nabla_{xy}f_{k}(x^t_k, y^{t}_k)-\nabla_{xy}f_{k}(x^t_k, y^*(x_k))\nabla_y f(x_k,y_k^{t+1})\\
  %      &+ \nabla_{xy}f_{k}(x^t_k, y^*(x_k))(\nabla_y f_k(x_k,y_k^{t+1}) 
 %       - [\nabla_{yy} f_k(x_k,y_k^*(x_k))]^{-1}\nabla_y f_k(x_k, y_k^*(x_k))\}\\
  %       &\quad= [\nabla_x f_k(x_t,y_t) - \nabla_x f_k(x_t,y_t^*)] - \%{(\nabla_{xy}f_{k}(x^t_k, y^{t}_k)-\nabla_{xy}f_{k}(x^t_k, y^*(x_k))\nabla_y f(x_k,y_k^{t+1})\\
   %     &+ \nabla_{xy}f_{k}(x^t_k, y^*(x_k))[\nabla_{yy} f_k(x_k,y_k^*(x_k))]^{-1}(\nabla_{yy} f_k(x_k,y_k^*(x_k))\nabla_y f_k(x_k,y_k^{t+1}) - \nabla_y f_k(x_k, y_k^*(x_k))\}\\        
 \end{align}
where $(i)$ holds due to  $\nabla_y f(x^t,y^*(x^t))=0$. Then we obtain that
\begin{align}
    &\| \nabla \widehat{\Phi}(x^t;y^{t+1}) - \nabla\Phi(x^t)\|
    \nonumber \\
    &\overset{(i)}{\leq} (L_{f,1}+\alpha L_{f,1}^2)\|y^{t+1}-y^*(x^t)\| \nonumber \\
    &\overset{(ii)}{\leq}   (L_{f,1}+\alpha L_{f,1}^2)(1-\alpha\mu)^{\frac{1}{2}}\|y^{t}-y^*(x^t)\|\\
    &= A\|y^t-y^*(x^t)\| 
\end{align}
where $A=(L_{f,1}+\alpha L_{f,1}^2)(1-\alpha\mu)^{\frac{1}{2}}$,
$(i)$ holds because $\nabla_x f$ and $\nabla_y f$ are $L_{f,1}$ Lipschitz with $x,y$, and 
$(ii)$ holds due to Lemma~\ref{scdecent}.
\end{proof}

\begin{lemma}[Hypergradient descent]\label{hyperdecent} Define $A=(L_{f,1}+\alpha L_{f,1}^2)(1-\alpha\mu)^{\frac{1}{2}}$, we have
\begin{align} \label{telescopingsum}
    \frac{1}{T}\sum_{t=0}^{T-1}(\frac{1}{2}-\eta L_{\Phi})\|\nabla\Phi(x^t)\|^2
\leq \frac{\Phi(x_0)-\inf\Phi(x)}{\eta T} +\frac{1}{T}
    (\frac{1}{2}+\eta L_{\Phi})A^2\sum_{k=0}^{T-1}\|y^t-y^*(x^t)\|^2
\end{align}
\end{lemma}
\begin{proof}
      The proof is very similar to the proof of Theorem 1 in~\cite{ji2021bilevel}. The $L_{\Phi}$-smoothness of $\Phi(x)$  implies that
    \begin{align}
        \Phi(x^{t+1}) - \Phi(x^t)\leq \langle\nabla \Phi(x^t), x^{t+1} - x^t \rangle+ \frac{L_{\Phi}}{2}\|{x^{t+1}-x^t}\|^2
    \end{align}
Define $h^t=\nabla \widehat{\Phi}(x^t;y^{t+1})=\nabla_x f(x^t, y^{t+1})  
    - \alpha\nabla_{xy}f(x^t, y^{t})\nabla_y f(x^t, y^{t+1})$. We have
\begin{align}
    \Phi(x^{t+1}) \leq
    &\Phi(x^t)-\eta \langle\nabla\Phi(x^t),h^t\rangle+\frac{L_{\Phi}\eta^2}{2}\|h^t\|  \nonumber \\
    &\leq \Phi(x^t)-\eta(\frac{1}{2}-\frac{\eta L_{ \Phi}}{2})\|h^t\|^2 +\frac{\eta^2L_\Phi}{2}\|h^t-\nabla \Phi(x^t)\|^2 \nonumber\\
    &\leq \Phi(x^{t})-(\frac{\eta}{2}- \eta^2 L_{\Phi})\|\nabla\Phi(x^t)\|^2+(\frac{\eta}{2}+\eta^2 L_\Phi)\|h^t-\nabla \Phi(x^t)\|^2
\end{align}
Do telescoping and use Lemma~\ref{lm:hypergradient-bias} we get
\begin{align} 
    \frac{1}{T}\sum_{t=0}^{T-1}(\frac{1}{2}-\eta L_{\Phi})\|\nabla\Phi(x^t)\|^2
    \overset{\text{Lemma}~\ref{lm:hypergradient-bias}}{\leq} \frac{\Phi(x_0)-\inf\Phi(x)}{\eta T} +\frac{1}{T}
    (\frac{1}{2}+\eta L_{\Phi})A^2\sum_{k=0}^{T-1}\|y^t-y^*(x^t)\|^2
\end{align}

%Note that if the series $\lim_{T \to \infty}\sum_{t=0}^{T}\|y^t-y^*(x^t)\|^2$ converge, then the algorithm converge.\\
%Intuitively, $\|y_k^t-y^*(x^t)\|^2$ is a contract geometric series and would converge. So we try to track the change of $y^t$.
%So introduce optimaly measure $V_k=\frac{1}{\alpha^2}\|x_{k+1}-x_k\|^2+\|h_k-\nabla\Phi(x_{k})\|^2$
\end{proof}

\begin{lemma}[Lower Level Convergence]\label{lemma:lowerlevel}
$\|y^{t+1}-y^*(x^{t+1})\|^2\leq C\|y^t-y^*(x^t)\|^2+D\|\nabla\Phi(x^t)\|^2$, where $C=1-\alpha^2\mu^2+2(1+\frac{1}{\alpha\mu})\frac{L_{f,1}^2}{\mu^2}\eta^2
A^2$, $D=2(1+\frac{1}{\alpha\mu})\eta^2\frac{L_{f,1}^2}{\mu^2}$.
\end{lemma}
\begin{proof}
Note that
\begin{align}\label{ydecent}
   & \|y^{t+1}-y^*(x^{t+1})\|^2 \nonumber\\
   &\overset{(i)}{\leq} (1+\alpha\mu)\|y^{t+1}-y^*(x^t)\|^2+(1+\frac{1}{\alpha\mu})\|y^*(x^{t+1})-y^*(x^{t})\|^2 \nonumber\\
%    &\leq (1+\alpha\mu)^{\frac{1}{2}}\|y_{k}^{t+1}-y^*(x_{k})\|^2+[(1+\frac{1}{\alpha\mu})L_y]^{\frac{1}{2}}\|x_{k+1}-x_k\|\\
   & \overset{(ii)}{\leq} (1+\alpha\mu)(1-\alpha\mu)\|y^{t}-y^*(x^t)\|^2+(1+\frac{1}{\alpha\mu})\frac{L_{f,1}^2}{\mu^2}\|x^{t+1}-x^t\|^2 \nonumber\\
   &\leq (1+\alpha\mu)(1-\alpha\mu)\|y^{t}-y^*(x^t)\|^2+
   (1+\frac{1}{\alpha\mu})\frac{2L_{f,1}^2}{\mu^2}\eta^2(\|h^t-\nabla\Phi(x^t)\|^2+\|\nabla\Phi(x^t)\|^2) \nonumber\\  
   &= C\|y^t-y^*(x^t)\|^2+D\|\nabla\Phi(x^t)\|^2,
\end{align}
where $(i)$ uses the Young's inequality, $(ii)$ is due to Lemma~\ref{scdecent} and the Lipschitzness of the mapping $y^*(x)$,  $C=1-\alpha^2\mu^2+2(1+\frac{1}{\alpha\mu})L_y^2\eta^2
A^2$; $D=2(1+\frac{1}{\alpha\mu})\eta^2\frac{L_{f,1}^2}{\mu^2}$.
\end{proof}

\begin{proof}[Proof of Theorem~\ref{mainthm1}]
Substituting Lemma~\ref{lemma:lowerlevel} to Lemma~\ref{hyperdecent} yields
\begin{align}\label{paratelescopingsum}
    &\frac{1}{T}\sum_{t=0}^{T-1}\left[\frac{1}{2}-\eta L_{\Phi}-(\frac{1}{2}+\eta L_{\Phi})A^2D\right]\|\nabla\Phi(x^t)\|^2 \nonumber\\
    &\leq \frac{\Phi(x^0)-\inf\Phi(x)}{\eta T} + \frac{1}{T}
    (\frac{1}{2}+\eta L_{\Phi})A^2\frac{\|y^0-y^*(x^0)\|^2}{1-C},
\end{align}
where  $A=(L_{f,1}+\alpha L_{f,1}^2)(1-\alpha\mu)^{\frac{1}{2}}$, $C=1-\alpha^2\mu^2+2(1+\frac{1}{\alpha\mu})\frac{L_{f,1}^2}{\mu^2}\eta^2
A^2$; $D=2(1+\frac{1}{\alpha\mu})\eta^2\frac{L_{f,1}^2}{\mu^2}$.

We want to carefully choose the parameter $\alpha, \eta$ s.t. $C<1$, $\alpha \leq \frac{1}{L_{f,1}}$ and $\frac{1}{2}-\eta L_{\Phi}-(\frac{1}{2}+\eta L_{\Phi})A^2D >0$. For example, we can choose 
 % \begin{equation}\label{parachoice}
 % \alpha=\frac{1}{4L_{f,1}}, \eta=\min\{ \frac{\mu^2}{5L_{f,1}^3\sqrt{(\frac{4L_{f,1}}{\mu}-\frac{\mu}{4L_{f,1}})}}; \frac{1}{8L_\Phi}; \sqrt{\frac{1}{16N}}; \sqrt[3]{\frac{1}{81NL_\Phi}}\}
 % \end{equation}
 % while
 % \begin{equation}
 % N=\frac{25L_{f,1}^4(\frac{4L_{f,1}}{\mu}+1)}{16\mu^2}. \nonumber
 % \end{equation}
$\alpha=\frac{1}{4L_{f,1}}, \eta=\min\left(\frac{\mu^2}{5L_{f,1}^3\sqrt{(\frac{4L_{f,1}}{\mu}-\frac{\mu}{4L_{f,1}})}}, \frac{1}{8L_\Phi}, \sqrt{\frac{1}{16N}}, \sqrt[3]{\frac{1}{81NL_\Phi}}\right)$, and $N=\frac{25L_{f,1}^4(\frac{4L_{f,1}}{\mu}+1)}{16\mu^2}$.
 
\end{proof}

\section{Theoretical Analysis: An Example on Multivariate Linear Regression}\label{sec:analysis_rank_learning}

In this section, we provide the theoretical analysis to demonstrate why our method is able to learn the ground truth rank, whereas HETLoRA fails in a multivariate linear regression example

% To clarify this mechanism of ``this automatic rank adaptation of PF2LoRA", we first construct a multivariate linear regression example and provide a theoretical analysis to demonstrate why our method can accurately learn the ground truth rank, whereas HETLoRA fails. Then we conduct a synthetical experiment to compare two algorithms in federated learning with two clients. The experimental results confirm that our algorithm is able to learn the ground truth ranks for two clients and converge to the optimal solution. In contrast, HETLoRA underestimates the initial rank of some clients due to random rank initialization strategy, resulting in underfitting and suboptimal performance in such clients.

% \subsection{Theoretical Analysis}
If our algorithm can find a better low rank approximation than HETLoRA, then our method surely performs better than HETLoRA. 
So theoretically, we want to find the exact analytic solution of the best low rank approximation.
Recall multivariate linear regression problem, the goal is to minimize the reconstruction error:
\[
\min_{W\in \mathbb{R}^{m\times n}} \|Y-XW\|^2_F
\]
where $(X,Y)$ is the data and label.
We know the solution which can minimize the reconstruction error is,
\[ W=(X^TX)^{-1}X^TY\]

However, $rank(W)$ is possibly very large, leading to computationally inefficient. So we want to find the optimal low-rank matrix approximation of $W$
(i.e. minimize the reconstruction error with small rank of $W$), then we add a rank restriction on $W$,
\[
Y=XW+\epsilon, \quad \text{s.t.}, \quad rank(W)\leq r.
\]

In statistics, this is a Reduced Rank Regression (RRR) problem, which has been well-explored,
\[
\min_{W\in \mathbb{R}^{m\times n}} \|Y-XW\|_F^2, \quad \text{s.t.}, \quad rank(W)\leq r,
\]
which is equivalent to 
\[
\min_{W\in \mathbb{R}^{m\times n}} tr[(Y-XW)(Y-XW)^T], \quad rank(W)\leq r
\]
where $tr(.)$ is the matrix trace. \\
Given the upper bound of $rank(W)=r$, we directly do rank factorization on $W$, i.e., LoRA:
\[
\min_{A \in \mathbb{R}^{m \times r}, B \in \mathbb{R}^{r \times n}} tr[(Y-XAB)(Y-XAB)^T], 
\]

Specifically in HETLora setting, given the rank initialization of the $k-$client: $r_k^{init}$, the objective function is:
\[
\min_{A \in \mathbb{R}^{m \times r_k^{init}}, B \in \mathbb{R}^{r_k^{init} \times n}} tr[(Y_k-X_kAB)(Y_k-X_kAB)^T].
\]
In our setting, we initialize the rank of the common adapter to $r$, and the local adapter to $\tilde{r}$, the objective function is, 
\[
\min_{A\in\mathbb{R}^{m \times r}, B \in \mathbb{R}^{r \times n},C_k \in \mathbb{R}^{m \times \tilde{r}},D_k \in \mathbb{R}^{\tilde{r} \times n}} tr[(Y_k-X_k(AB+C_kD_k))(Y_k-X_k(AB+C_kD_k)^T]. 
\]

In the synthetic experiment, we make global $AB$ to be in the orthogonal row and vector space of $C_kD_k$, then we directly get \[r(W_k)=r(AB+C_kD_k)=r(AB)+r(C_kD_k)=r+\tilde{r}\] 
then our problem is equivalent to reduced-rank regression problem.

\begin{lemma} \citep{reinsel1998multivariate}  Theorem 2.2[RRR solution]
Suppose the \((m+n)\)-dimensional random vector \((Y_k, X_k)\) has mean vector \(0\) and covariance matrix with:
\[
\Sigma_{yx} = \Sigma_{xy} = \text{Cov}(Y_k, X_k) \quad \text{and} \quad \Sigma_{xx} = \text{Cov}(X_k) \quad \text{nonsingular}.
\]

Then, for any positive-definite matrix \(\Sigma\), an \(m \times r\) matrix \(A\) and \(r \times n\) matrix \(B\), for \(r \leq \min(m, n)\), which minimize
\[
\text{tr} \big\{ \mathbb{E} \big[\Sigma^{1/2} (Y_k - X_kAB)(Y_k - X_kAB)^\top \Sigma^{1/2} \big] \big\}
\]
are given by:
\[
A^{(r)} = \Sigma^{-1/2} [V_1, \dots, V_r] = \Sigma^{-1/2} V, \quad B^{(r)} = V^\top \Sigma^{1/2} \Sigma_{yx} \Sigma_{xx}^{-1}
\]
where \(V = [V_1, \dots, V_r]\) and \(V_j\) is the (normalized) eigenvector that corresponds to the \(j\)-th largest eigenvalue \(\lambda_j^2\) of the matrix:
\[
\Sigma^{1/2} \Sigma_{yx} \Sigma_{xx}^{-1} \Sigma_{xy} \Sigma^{1/2}, \quad j = 1, 2, \dots, r.
\]
\end{lemma}
From solution formula we directly get minimum truncated error
\[
\min_{A,B:rank(AB) \leq r}\|W-AB\|_F^2=\sqrt{\sum_{i=r+1}^n \lambda_i} \quad \forall W,rank(W) \geq r
\]

\subsubsection{Low-Rank approximation}
Specifically in HETLoRA setting, given the rank initialization of the $k-$client: $r_k^{init}$, the objective function is:
\[
\min_{A \in \mathbb{R}^{m \times r_k^{init}}, B \in \mathbb{R}^{r_k^{init} \times n}} tr[(Y_k-X_kAB)(Y_k-X_kAB)^T].
\]
In our setting, we initialize the rank of the common adapter to $r$, and the local adapter to $\tilde{r}$, the objective function is, 
\begin{equation} \label{eq:obj_mlr}
    \min_{A\in\mathbb{R}^{m \times r}, B \in \mathbb{R}^{r \times n},C_k \in \mathbb{R}^{m \times \tilde{r}},D_k \in \mathbb{R}^{\tilde{r} \times n}} tr[(Y_k-X_k(AB+C_kD_k))(Y_k-X_k(AB+C_kD_k)^T]. 
\end{equation}

note $C_kD_k$, is a local adapter. we mark 
\[W_k=P_kQ_k=AB+C_kD_k\]
 note that $rank(P_kQ_k) \in [r-\tilde{r}, r+\tilde{r}]$.
Generally we cannot say the problem (\ref{eq:obj_mlr}) 
and 

\[
\min_{P_k \in \mathbb{R}^{m \times r+\tilde{r}} \quad Q_1 \in \mathbb{R}^{r+\tilde{r} \times n}} tr[(Y_k-X_kP_kQ_k)(Y_k-X_kP_kQ_k)]^T,
\]
are equivalent since the former one is subset of the latter problem. However, under some certain dataset setting, the two problems are equivalence. We defer the equivalence proof to Lemma \ref{sec:eq_prob}.

Suppose we have two clients, the optimal solution in HETLoRA is 
\[
\text{Client 1} \quad A_1^{r_1^{init}} = \Sigma^{-1/2} [V_1, \dots, V_{r_{init}}] = \Sigma^{-1/2} V, \quad B_1^{r_1^{init}} = V^\top \Sigma^{1/2} \Sigma_{yx} \Sigma_{xx}^{-1}
\]
\[
\text{Client 2} \quad A_2^{r_2^{init}} = \Sigma^{-1/2} [V_1, \dots, V_{r_{init}}] = \Sigma^{-1/2} V, \quad B_2^{r_2^{init}} = V^\top \Sigma^{1/2} \Sigma_{yx} \Sigma_{xx}^{-1}
\]
In our setting, the optimal solution is 
\[
\text{Client 1} \quad P_1^{r+\tilde{r_1}} = \Sigma^{-1/2} [V_1, \dots, V_{r+\tilde{r_1}}] = \Sigma^{-1/2} V, \quad Q_1^{r+\tilde{r_1}} = V^\top \Sigma^{1/2} \Sigma_{yx} \Sigma_{xx}^{-1}
\]
\[
\text{Client 2} \quad P_2^{r+\tilde{r_2}} = \Sigma^{-1/2} [V_1, \dots, V_{r+\tilde{r_2}}] = \Sigma^{-1/2} V, \quad Q_1^{r+\tilde{r_1}} = V^\top \Sigma^{1/2} \Sigma_{yx} \Sigma_{xx}^{-1}
\]

Suppose for Client 1 data, $\Sigma^{1/2} \Sigma_{yx} \Sigma_{xx}^{-1} \Sigma_{xy} \Sigma^{1/2}$ has eigenvector $\lambda_1=\lambda_2=\lambda_3=1$; $\lambda_4=\dots=\lambda_n=0$, obviously the low rank approximation is $r_1^*=3$.
For Client 2 data, $\Sigma^{1/2} \Sigma_{yx} \Sigma_{xx}^{-1} \Sigma_{xy} \Sigma^{1/2}$ has eigenvector $\lambda_1=\dots=\lambda_4=1$; $\lambda_5=\dots=\lambda_n=0$, the low rank approximation is $r_2^*=4$.

In our synthetic experiments \ref{sec:automatic_rank_learning},
HETLoRA underestimates the rank for client 1, i.e., $r_1^{init}=2 < r_1^*=3$ due to the random rank initialization, and the learned rank $r_1=1$ by self-pruning; Client 2 initializes a reasonable $r_2^{init}=10$, and the learned rank $r_2=5=r_2^*$.
Thus client 1 fails to learn the optimal low rank approximation because 
\[\min_{A,B:rank(AB) \leq r_1^{init}}\|W-AB\|_F^2=\sqrt{\sum_{i=r+1}^n \lambda_i}=1.\]
Our PF2LoRA initializes $r=4$ for the common adapter ($AB$), and $\tilde{r}=2$ ($C_kD_k$) for the local adapter, which means $r-\tilde{r} = 2 \leq rank(AB+C_kD_k) \leq r+\tilde{r} = 6$, and learned rank for client 1 is $r_1=3$. The learned rank for client 2 is $r_2=4$. Both succeeded to learn the optimal low rank approximation.
\[\min_{A,B, C_k, D_k:r-\tilde{r}\leq rank(W_k) \leq r+\tilde{r}}\|W-W_k\|_F^2=\sqrt{\sum_{i=r+\tilde{r}}^n \lambda_i}=0.\]

\subsubsection{Problem Equivalence} \label{sec:eq_prob}
Next we prove two problems to be equivalent:
\[
\min_{A \in \mathbb{R}^{m \times r} \quad B \in \mathbb{R}^{r \times n} \quad C_k \in \mathbb{R}^{m \times \tilde{r}} \quad D_k \in \mathbb{R}^{\tilde{r} \times n}} tr[(Y_k-X_k(AB+C_kD_k))(Y_k-X_k(AB+C_kD_k)^T] 
\]
and 
\[
\min_{W_k\in\mathbb{R}^{m\times n}} tr[(Y_k-X_kW_k)(Y_k-X_kW_k)^T], \quad r-\tilde{r} \leq rank(W_k)\leq r+\tilde{r}
\]
\begin{lemma}
The rank of the sum of \(AB\) and \(CD\) satisfies:
\[
r(AB + CD) = r(AB) + r(CD)
\]
if and only if 
\[
\text{dim}(\mathcal{C}_1 \cap \mathcal{C}_2) = \text{dim}(\mathcal{R}_1 \cap \mathcal{R}_2) = 0.
\]
where \(\mathcal{C}_1\) and \(\mathcal{C}_2\) be the column spaces of \(AB\) and \(CD\), and  \(\mathcal{R}_1\), \(\mathcal{R}_2\) are their row spaces. 
\end{lemma}
\begin{proof}
To simplify the notation in proof, we mark $c=\text{dim}(\mathcal{C}_1 \cap \mathcal{C}_2)$, $d= \text{dim}(\mathcal{R}_1 \cap \mathcal{R}_2)$; $E=AB$, $F
=CD$. First, the condition \(c = d = 0\) is necessary, as two strings of inequalities show:
\[
r(E + F) \leq r[(E, F)] = r(E) + r(F) - c \leq r(E) + r(F),
\]
\[
r(E + F) \leq r[(E; F)] = r(E) + r(F) - d \leq r(E) + r(F).
\]

To show \(c = d = 0\) is sufficient, we use full-rank decompositions of \(E\) and \(F\):
\[
E = C_1R_1, \quad r(A) = r(C_1) = r(R_1) = a,
\]
where \(E\) is \(m \times n\), \(C_1\) is \(m \times a\), and \(R_1\) is \(a \times n\).

\[
F = C_2R_2, \quad r(F) = r(C_2) = r(R_2) = b,
\]
where \(F\) is \(m \times n\), \(C_2\) is \(m \times b\), and \(R_2\) is \(b \times n\).

Such representations exist since \(R_1\) can be any matrix whose rows form a basis of the row space of \(A\). Then \(A = C_1R_1\) for some \(C_1\), and:
\[
r(E) = r(C_1) = \min\big(\text{rank}(C_1), \text{rank}(R_1)\big) \leq a = r(E).
\]

We now write:
\[
E + F = C_1R_1 + C_2R_2 = (C_1, C_2) \begin{pmatrix} R_1 \\ R_2 \end{pmatrix} = CR,
\]
 Then \(c = 0\) implies that all the \(a + b\) columns of \(C\) are linearly independent, and so \(C\) has a left inverse \(L\) such that \(LC = I\). Thus, when \(c = 0\),
\[
r(E + F) = r(CR) \geq r(LCR) = r(R) = r(E) + r(F) - d.
\]

If in addition \(d = 0\), the entire string collapses, and:
\[
r(E + F) = r(E) + r(F).
\]
\end{proof}
In the following synthetic experiment setting we make global $AB$ in orthogonal row and vector space
of $C_1D_1$, $C_2D_2$, according to above lemma we directly get \[r(W_1)=r(AB+C_1D_1)=r(AB)+r(C_1D_1)=r+\tilde{r}\] and
\[r(W_2)=r(AB+C_2D_2)=r(AB)+r(C_2D_2)=r+\tilde{r}\]
So under our synthetic experiment setting, our problem is equivalent to reduced-rank regression problem, which provides a theoretical guarantee.

\section{Experimental Settings in the Synthetic Example}\label{sec:hyperparameter_syn_exp}

We conduct a synthetic experiment of multivariate linear regression in federated learning to show why HETLoRA fails to learn the ground truth rank, but PF2LoRA does. The following describes the details of experiments and the hyperparameter settings for both algorithms,

        \begin{enumerate}
            \item HETLoRA: Following its rank initialization strategy $r_{min}\leq rank_1 \leq rank_2 ...\leq rank_k...\leq r_{max}$, we assume that $r_{min}=1, r_{max} = 12$ and initialize $\hat{W}_k = \hat{A}_k\hat{B}_k$ by,
            \begin{align*}
                & \hat{A}_1\in \mathbb{R}^{10\times 2}, \hat{B}_1\in \boldsymbol{0}^{2\times 10}, \ s.t. \ \hat{A}_1 \sim  \mathcal{N}(0, 1),\\
                & \hat{A}_2\in \mathbb{R}^{10\times 10}, \hat{B}_2\in \boldsymbol{0}^{10\times 10}, \ s.t. \ \hat{A}_2 \sim  \mathcal{N}(0, 1)
            \end{align*}
            so we have $rank(\hat{A}_1)=2$ and $rank(\hat{A}_2)=10$. We can easily get that the total number of trainable parameters for two clients is $240$. Other hyperparameters are set as follows. We search the regularization factor $\gamma$ in the range $[0.05, 0.5\}$ with the search grid $0.05$ and set it to the optimal value $0.1$. The pruning parameter $\gamma=0.3$, which is responsible for imposing the regularization to the last $30\%$ columns to sparse them. We tune the learning rate within the range $\{0.001, 0.002, 0.003, 0.004, 0.005\}$ and set it to the optimal value $0.002$.  The total training steps are 2000, and the communication is performed every 10 steps, which means we train the parameters for 10 steps locally, and then execute the parameter aggregation and distribution.  
            
            \item PF2LoRA: For a fair comparison, we initialize the trainable parameters $\hat{W}_k = \hat{A}_k\hat{B}_k+\hat{C}_k \hat{D}_k$, and make sure the total number of trainable parameters to be the same as that in HETLoRA.
            For client $k=1, 2$, we have $r=4, \tilde{r}=2$ and,
            \begin{align*}
                &\hat{A}_k\in \mathbb{R}^{10\times 4}, \hat{B}_k\in \mathbb{R}^{4\times 10}, \hat{C}_k\in \mathbb{R}^{10\times 2}, \hat{D}_k\in \mathbb{R}^{2\times 10}, \\
                & s.t. \ \hat{A}_k \sim  \mathcal{N}(0, 1), \hat{C}_k \sim  \mathcal{N}(0, 1), \hat{C}_k \sim  \mathcal{N}(0, 1), \hat{D}_k \sim  \mathcal{N}(0, 1).
            \end{align*}
            and $A_kB_k$ is orthogonal to the matrix $C_kD_k$, such that their column space or row space are independent mutually.
            The total number of training steps are fixed as 2000, and the communication interval is 10. We search the best upper-level and lower-level learning rates within the range $[0.001, 0.01]$ with the search grid  of $0.001$, and set the best upper-level learning rate to $0.005$ and the lower-level learning rate to $0.002$. In each communication round, we aggregate the common adapter parameters $A_k, B_k$ and then distribute them, and the local adapter parameters  $C_k, D_k$ are not involved in communication.   
        \end{enumerate}
         
    % \begin{figure}[!h]
    %     \centering
    %     \subfigure[HETLoRA fails to converge to the ground truth.]{\includegraphics[width=1\linewidth]{figures/HETLoRA_simulation.pdf}}
    %     \subfigure[PF2LoRA can converge to the ground truth.]{\includegraphics[width=1\linewidth]{figures/pf2lora_simulation.pdf}}
    %     \caption{Comparison of two algorithms. Left to right: the training loss on two clients, the testing loss on two clients, Frobenius norm distance $\|W_k - W_k^*\|_F$, $k = 1, 2$, and the rank evolution of two clients.}
    %     \label{fig:synthetic_exp}
    % \end{figure}

\section{Sensitivity Analysis of Hyperparameter} \label{sec:sensitivity}

We run our algorithm PF2LoRA on GLUE benchmarks using a hyperparameter sweep, and the results are presented in Figure \ref{fig:sensitity_parameters}. In our setting, we require the local adapter to be light-weight, so the rank of local adapters is always small, i.e., $\tilde{r}=2$. We perform a hyperparameter sweep on the local learning rate $\alpha$ and the rank of the common adapter, respectively. As you see in subfigure 2(a), our algorithm is pretty robust to the learning rate $\alpha$. Since COLA dataset is more challenging than others, a larger rank is helpful to improve the model performance, but the performance keeps almost the same when the rank is larger than 8. Our algorithm also exhibits high robustness on data MNLI and SST-2.     

\begin{figure}[!t]
    \centering
    \subfigure[ PF2LoRA performance vs. learning rate $\alpha$.]{\includegraphics[width=0.35\linewidth]{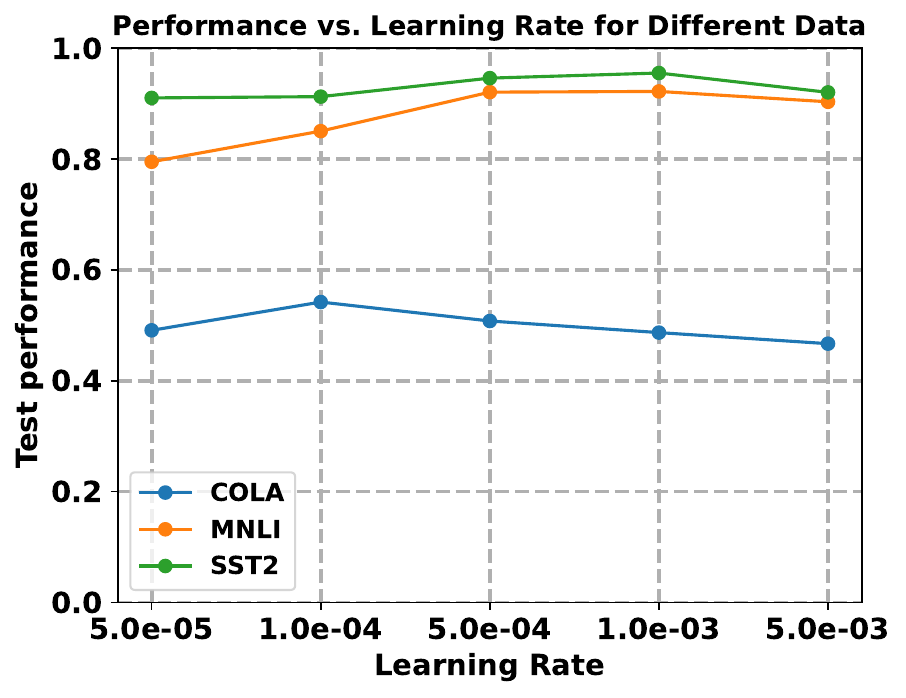}}\ \ \
    \subfigure[ PF2LoRA performance vs. rank $r$.]{\includegraphics[width=0.35\linewidth]{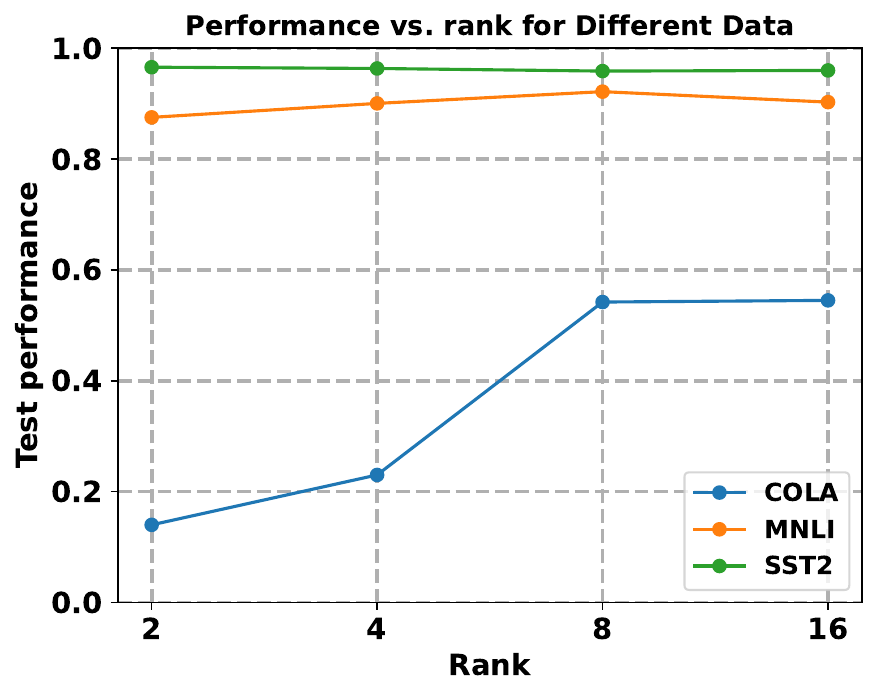}}
    \caption{ Sensitivity analysis of hyperparameters.}
    \label{fig:sensitity_parameters}
\end{figure}

\section{Computation and Communication Cost} \label{sec:computing_cost}

We evaluated the total computational costs (FLOPs on 8 NVIDIA RTX A6000 GPUs) and communication costs in a single communication round for each algorithm on GLUE benchmark. The results are summarized in Table \ref{tab:computing_cost}.  From our understanding, communication costs are the total number of parameters that participate in the aggregation and distribution of parameters in federated learning. The computational cost (FLOPs) per round are determined by the number of model parameters and the forward/backward propagation operations. As PF2LoRA requires to compute the hessian-vector product for hypergradient estimation, it incurs a higher computational cost. But the communication cost of PF2LoRA remains consistent with that of HOMLoRA and Centralized LoR, as the communication parameters in PF2LoRA are only global adapters that have the same rank $r_k=8$ with that in HOMLoRA and Centralized LoRA. Instead, HETLoRA has a higher parameter rank requirement for a high performance, resulting in increased communication costs.    

\begin{table*}[htbp]
  \centering
  \caption{ Computational/Communication costs per communication round.}
    \setlength{\tabcolsep}{10pt}
    \begin{tabular}{l|cc}
    % \hline
    \Xhline{1.2pt}
    Method  & TFLOPs/round  & Communication parameters/round \\
    \hline
    Centralized LoRA  ($r_k=8$) & 258.40  & 0.30M \\
    % \hline
    HOMLoRA ($r_k=8$) & 258.40  & 0.30M\\
    % \hline
    Per-FedAvg-LoRA ($r_k=8$) & 908.00 & 0.30M \\
    % \hline
    HETLoRA ($r_{max}=12, r_{min}=8$) & 272.60  & 0.35M \\
    % \hline
    PF2LoRA ($r_k=8, \tilde{r}=2$)  & 1202.40    & 0.30M \\
    % \hline
    \Xhline{1.2pt}
    \end{tabular}%
  \label{tab:computing_cost}%
  % \vspace{-0.2in}
\end{table*}%
%%%%%%%%%%%%%%%%%%%%%%%%%%%%%%%%%%%%%%%%%%%%%%%%%%%%%%%%%%%%%%%%%%%%%%%%%%%%%%%
%%%%%%%%%%%%%%%%%%%%%%%%%%%%%%%%%%%%%%%%%%%%%%%%%%%%%%%%%%%%%%%%%%%%%%%%%%%%%%%

\end{document}